\documentclass{article}

\usepackage{amsmath,amsfonts,bm}

\def\eqref#1{equation~\ref{#1}}

\def\1{\bm{1}}

\DeclareMathAlphabet{\mathsfit}{\encodingdefault}{\sfdefault}{m}{sl}
\SetMathAlphabet{\mathsfit}{bold}{\encodingdefault}{\sfdefault}{bx}{n}

\def\gA{{\mathcal{A}}}

\def\gF{{\mathcal{F}}}

\def\gP{{\mathcal{P}}}

\def\gR{{\mathcal{R}}}
\def\gS{{\mathcal{S}}}
\def\gT{{\mathcal{T}}}

\def\sN{{\mathbb{N}}}

\def\sP{{\mathbb{P}}}

\def\sR{{\mathbb{R}}}

\newcommand{\sE}{\mathbb{E}}

\usepackage{microtype}
\usepackage{graphicx}
\usepackage{booktabs} 

\usepackage{hyperref}
\usepackage{subfig}
\usepackage{enumitem}

\usepackage[accepted]{icml2024}

\usepackage{amsmath}
\usepackage{amssymb}
\usepackage{mathtools}
\usepackage{amsthm}

\usepackage[capitalize,noabbrev]{cleveref}

\theoremstyle{plain}
\newtheorem{theorem}{Theorem}[section]

\newtheorem{lemma}[theorem]{Lemma}

\theoremstyle{definition}

\newtheorem{condition}[theorem]{Condition}
\theoremstyle{remark}

\usepackage[textsize=tiny]{todonotes}

\icmltitlerunning{Single-Trajectory DRRL}

\begin{document}

\twocolumn[
\icmltitle{Single-Trajectory Distributionally Robust Reinforcement Learning}



\icmlsetsymbol{equal}{*}

\begin{icmlauthorlist}
\icmlauthor{Zhipeng Liang}{equal,hkustieda}
\icmlauthor{Xiaoteng Ma}{equal,tsing}
\icmlauthor{Jose Blanchet}{stanford}
\icmlauthor{Jun Yang}{tsing}
\icmlauthor{Jiheng Zhang}{hkustieda,hkustmath}
\icmlauthor{Zhengyuan Zhou}{nyu,arena}
\end{icmlauthorlist}

\icmlaffiliation{hkustieda}{Department of Industrial Engineering and Decision Analytics, Hong Kong University of Science and Technology}
\icmlaffiliation{hkustmath}{Department of Mathematics, Hong Kong University of Science and Technology}
\icmlaffiliation{stanford}{Department of Management Science and Engineering, Stanford University}
\icmlaffiliation{tsing}{Department of Automation, Tsinghua University}
\icmlaffiliation{nyu}{Stern School of Business, New York University}
\icmlaffiliation{arena}{Arena Technologies}

\icmlcorrespondingauthor{Zhengyuan Zhou}{zhengyuanzhou24@gmail.com}

\icmlkeywords{Robust Reinforcement Learning, Distributionally Robust Optimization, Single-Trajectory}

\vskip 0.3in
]
\printAffiliationsAndNotice{\icmlEqualContribution} 



\begin{abstract}
    To mitigate the limitation that the classical reinforcement learning (RL) framework heavily relies on identical training and test environments, Distributionally Robust RL (DRRL) has been proposed to enhance performance across a range of environments, possibly including unknown test environments.
    As a price for robustness gain, DRRL involves optimizing over a set of distributions, which is inherently more challenging than optimizing over a fixed distribution in the non-robust case.
    Existing DRRL algorithms are either model-based or fail to learn from a single sample trajectory.
    In this paper, we design a first fully model-free DRRL algorithm, called \emph{distributionally robust Q-learning with single trajectory (DRQ)}.
    We delicately design a multi-timescale framework to fully utilize each incrementally arriving sample and directly learn the optimal distributionally robust policy without modeling the environment, thus the algorithm can be trained along a single trajectory in a model-free fashion.
    Despite the algorithm's complexity, we provide asymptotic convergence guarantees by generalizing classical stochastic approximation tools. 
    Comprehensive experimental results demonstrate the superior robustness and sample complexity of our proposed algorithm, compared to non-robust methods and other robust RL algorithms.
\end{abstract}

\section{Introduction}
Reinforcement Learning (RL) is a machine learning paradigm for studying sequential decision problems.
Despite considerable progress in recent years \citep{silver2016mastering,mnih2015human,vinyals2019grandmaster}, 
RL algorithms often encounter a discrepancy between training and test environments. This discrepancy is widespread since test environments may be too complex to be perfectly represented in training, or the test environments may inherently shift from the training ones, especially in certain application scenarios, such as financial markets and robotic control.
Overlooking the mismatch could impede the application of RL algorithms in real-world settings, given the known sensitivity of the optimal policy of the Markov Decision Process (MDP) to the model \citep{mannor2004bias,iyengar2005robust}.

To address this concern,
Distributionally Robust RL (DRRL) \citep{zhou2021finitesample,yang2021theoretical,shi2022distributionally, panaganti2022sample, panaganti2022robust, ma2022distributionally,Yang2018WassersteinDR,abdullah2019wasserstein, Neufeld2022RobustQA} formulates the decision problem under the assumption that the test environment varies but remains close to the training environment.
The objective is to design algorithms optimizing the worst-case expected return over an ambiguity set encompassing all possible test distributions. 
Evaluating a DRRL policy necessitates deeper insight into the transition dynamics than evaluating a non-robust one, as it entails searching for the worst-case performance across all distributions within the ambiguity set.
Therefore, most prior solutions are model-based, require the maintenance of an estimator for the entire transition model and the ambiguity set.
Such requirements may render these algorithms less practical in scenarios with large state-action spaces or where adequate modeling of the real environment is unfeasible.

Prompted by this issue, we study a fully model-free DRRL algorithm in this paper, which learns the optimal DR policy without explicit environmental modeling.
The algorithm's distinctive feature is its capacity to learn from a single sample trajectory, representing the least demanding requirement for data collection.
This feature results from our innovative algorithmic framework, comprising incrementally updated estimators and a delicate approximation scheme.
While most model-free non-robust RL algorithms support training in this setting—contributing to their widespread use—no existing work can effectively address the DRRL problem in this way.
The challenge arises from the fact that approximating a DR policy by learning from a single trajectory suffers from restricted control over state-action pairs and limited samples, i.e., only one sample at a time.
As we will demonstrate, a simple plug-in estimator using one sample, which is unbiased in the non-robust $Q$-learning algorithm, fails to approximate any robust value accurately.

The complexity of this task is further affirmed by the sole attempt to develop a model-free DRRL algorithm in \citep{liu2022distributionally}.
It relies on a restricted simulator assumption, enabling the algorithm to access an arbitrary number of samples from any state-action pair, thereby amassing sufficient system dynamics information before addressing the DRRL problem. 
Relaxing the dependence on a simulator and developing a fully model-free algorithm capable of learning from a single trajectory necessitates a delicate one-sample estimator for the DR value, carefully integrated into an algorithmic framework to eradicate bias from insufficient samples and ensure convergence to the optimal policy.
Moreover, current solutions heavily depend on the specific divergence chosen to construct the ambiguity set and fail to bridge different divergences, underscoring the practical importance of divergence selection.

Thus a nature question arises: 
\textit{Is it possible to develop a model-free DRRL framework that can learn the optimal DR policy across different divergences using only a single sample trajectory for learning?}

\subsection{Our Contributions}
In this paper, we provide a positive solution to the aforementioned question by making the following contributions:

\begin{enumerate}[leftmargin=*]
    \item We introduce a pioneering approach to construct the ambiguity set using the Cressie-Read family of $f$-divergence. By leveraging the strong duality form of the corresponding distributionally robust reinforcement learning (DRRL) problem, we reformulate it, allowing for the learning of the optimal DR policies using misspecified MDP samples. This formulation effortlessly covers widely used divergences such as the Kullback-Leibler (KL) and $\chi^2$ divergence.
    \item To address the additional nonlinearity that arises from the DR Bellman equation, which is absent in its non-robust counterpart, we develop a novel multi-timescale stochastic approximation scheme. This scheme carefully exploits the structure of the DR Bellman operator. The update of the $Q$ table occurs in the slowest loop, while the other two loops are delicately designed to mitigate the bias introduced by the plug-in estimator due to the nonlinearity.
    \item We instantiate our framework into a DR variant of the $Q$-learning algorithm, called \emph{distributionally robust $Q$-learning with single trajectory (DRQ)}. This algorithm solves discount Markov Decision Processes (MDPs) in a fully online and incremental manner. We prove the asymptotic convergence of our proposed algorithm by extending the classical two-timescale stochastic approximation framework, which may be of independent interest.
    \item We conduct extensive experiments to showcase the robustness and sample efficiency of the policy learned by our proposed DR $Q$-learning algorithm.
    We also create a deep learning version of our algorithm and compare its performance to representative online and offline (robust) reinforcement learning benchmarks on classical control tasks.
\end{enumerate}

\subsection{Related Work}
\textbf{Robust MDPs and RL:} 
The framework of robust MDPs has been studied in several works such as \citet{nilim2005robust, iyengar2005robust, wiesemann2013robust, lim2013reinforcement,ho2021partial,goyal2022robust}. 
These works discuss the computational issues using dynamic programming with different choices of MDP formulation, as well as the choice of ambiguity set, when the transition model is known. 
Robust Reinforcement Learning (RL) \citep{roy2017reinforcement, badrinath2021robust, wang2021online} relaxes the requirement of accessing to the transition model by simultaneously approximating to the ambiguity set as well as the optimal robust policy, using only the samples from the misspecified MDP. 

\textbf{Online Robust RL:}
Existing online robust RL algorithms including \citet{wang2021online, badrinath2021robust,roy2017reinforcement}, highly relies on the choice of the $R$-contamination model and could suffer over-conservatism. 
This ambiguity set maintains linearity in their corresponding Bellman operator and thus inherits most of the desirable benefits from its non-robust counterpart.
Instead, common distributionally robust ambiguity sets, such as KL or $\chi^2$ divergence ball, suffer from extra nonlinearity when trying to learn along a single-trajectory data, which serves as the foundamental challenge in this paper. 

\textbf{Distributionally Robust RL:}
To tackle the over-conservatism aroused by probability-agnostic $R$-contamination ambiguity set in the aforementioned robust RL, DRRL is proposed by constructing the ambiguity set with probability-aware distance \citep{zhou2021finitesample,yang2021theoretical,shi2022distributionally, panaganti2022sample, panaganti2022robust, ma2022distributionally}, including KL and $\chi^2$ divergence.  
As far as we know, most of the existing DRRL algorithms fall into the model-based fashion, which first estimate the whole transition model and then construct the ambiguity set around the model. 
The DR value and the corresponding policy are then computed based upon them.
Their main focus is to understand the sample complexity of the DRRL problem in the offline RL regime, leaving the more prevalent single-trajectory setting largely unexplored.
\section{Preliminary}

\subsection{Discounted MDPs}
\label{subsec:discount_form}
Consider an infinite-horizon MDP $(\gS, \gA, \gamma, \mu, P, r)$ where $\gS$ and $\gA$ are finite state and action spaces with cardinality $S$ and $A$.
$P:\gS\times \gA \rightarrow \Delta_{S}$ is the state transition probability measure.
Here $\Delta_{S}$ is the set of probability measures over $\gS$.
 $r$ is the reward function and $\gamma$ is the discount factor.
We assume that $r:\gS\times \gA \rightarrow [0,1]$ is deterministic and bounded in $[0, 1]$. 
A stationary  policy $\pi:\gS \rightarrow \Delta_{A}$ maps, for each state $s$ to a probability distribution over the action set $\gA$ and induce a random trajectory $s_1,a_1,r_1,s_2,\cdots$, with $s_1\sim \mu$, $a_n = \pi(s_n)$ and $s_{n+1}\sim P(\cdot \lvert s_n, a_n) \coloneqq P_{s_n, a_n} $ for $n\in \sN^+$. 
To derive the policy corresponding to the value function, we define the optimal state-action function $Q^{\star}:\gS\times \gA \rightarrow \sR$ as the expected cumulative discounted rewards under the optimal policy,
$
 Q^{\star}(s,a) \coloneqq \sup_{\pi\in \Pi}\sE_{\pi, P} [\sum_{n=1}^{\infty} \gamma^{n-1} r(s_n,a_n)\lvert s_1 = s, a_1 = a ].
$
The optimal state-action function $Q^*$ is also the fixed point of the Bellman optimality equation,
\begin{align}
 \label{eq:non-robust-opt}
 Q^{\star}(s,a) = r(s,a) + \gamma \sE_{s'\sim P} [\max_{a'\in \gA} Q^{\star}(s',a') ].
\end{align}

\subsection{$Q$-learning}
\label{subsec:q-learning}
Our model-free algorithmic design relies on the $Q$-learning template, originally designed to solve the non-robust Bellman optimality equation (Equation~\ref{eq:non-robust-opt}). $Q$-learning is a model-free reinforcement learning algorithm that uses a single sample trajectory to update the estimator for the $Q$ function incrementally. Suppose at time $n$, we draw a sample $(s_n, a_n, r_n, s'_n)$ from the environment. 
Then, the algorithm updates the estimated $Q$-function following:
\begin{align*}
Q_{n+1}(s_n,a_n) &= (1-\alpha_n)Q_n(s_n, a_n) + \\
& \quad \alpha_n (r_n + \gamma \max_{a^{\prime}\in \gA}Q_n(s^{\prime}_{n}, a^{\prime})),
\end{align*}
Here, $\alpha_n>0$ is a learning rate. 
The algorithm updates the estimated $Q$ function by constructing a unbiased estimator for the true $Q$ value, i.e., $r_n + \gamma \max_{a^{\prime}\in \gA} Q_n(s^{\prime}_n, a^{\prime})$ using one sample.
\subsection{Distributionally Robust MDPs}
\label{subsec:drdMDPs}
DRRL learns an optimal policy that is robust to unknown environmental changes, where the transition model $P$ and reward function $r$ may differ in the test environment. 
To focus on the perturbation of the transition model, we assume no pertubation to the reward function.
Our approach adopts the notion of distributional robustness, where the true transition model $P$ is unknown but lies within an ambiguity set $\gP$ that contains all transition models that are close to the training environment under some probability distance $D$.
To ensure computational feasibility, we construct the ambiguity set $\gP$ in the $(s,a)$-rectangular manner, where for each $(s,a)\in \gS\times \gA$, we define the ambiguity set $\gP_{s,a}$ as,
\begin{align}
 \label{eq:sa_ambiguity_set}
 \gP_{s,a} \coloneqq \{P'_{s,a}:\Delta_{S}\lvert D(P'_{s,a}\lVert P_{s,a})\le \rho\}.
\end{align}
We then build the ambiguity set for the whole transition model as the Cartesian product of every $(s,a)$-ambiguity set, i.e., $\gP = \prod_{(s,a)\in \gS\times \gA} \gP_{s,a}$.
Given $\gP$, we define the optimal DR state-action function $Q^{\star}$ as the value function of the best policy to maximize the worst-case return over the ambiguity set,
\begin{equation*}
    \begin{aligned}
 & Q^{\operatorname{rob}, \star}(s,a) \coloneqq \\
 & \sup_{\pi\in \Pi}\inf_{P\in \gP} \sE_{\pi, P}[\sum_{n=1}^{\infty} \gamma^{n-1} r(s_n,a_n)\lvert s_1 = s, a_1 = a].
    \end{aligned}
\end{equation*}
Under the $(s,a)$-rectangular assumption, the Bellman optimality equation has been established by \citet{iyengar2005robust, xu2010distributionally},
\begin{small}
\begin{align}    
Q^{\operatorname{rob},\star}(s,a) =& \gT_k(Q^{\operatorname{rob},\star})(s,a)\notag \\
\coloneqq  & r(s,a) + \gamma \inf_{P\in \gP} \sE_{s'\sim P} [\max_{a'\in \gA}Q^{\operatorname{rob}, \star}(s',a')].
\label{eq:optimality}
\end{align}
\end{small}

For notation simplicity, we would ignore the superscript $\operatorname{rob}$.
\section{Distributonally Robust $Q$-learning with Single Trajectory}
\label{sec:discount}
This section presents a general model-free framework for DRRL. 
We begin by instantiating the distance $D$ as Cressie-Read family of $f$-divergence \citep{cressie1984multinomial}, which is designed to recover previous common choices such as the $\chi^2$ and KL divergence. 
We then discuss the challenges and previous solutions in solving the corresponding DRRL problem, as described in Section~\ref{subsec:bias}. Finally, we present the design idea of our three-timescale framework and establish the corresponding convergence guarantee.
\subsection{Divergence Families}
\label{subsec:divergence}
Previous work on DRRL has mainly focused on one or several divergences, such as KL, $\chi^2$, and total variation (TV) divergences. 
In contrast, we provide a unified framework that applies to a family of divergences known as the Cressie-Read family of $f$-divergences.
This family is parameterized by $k\in(-\infty, \infty)\slash \{0,1\}$, and for any chosen $k$, the Cressie-Read family of $f$-divergences is defined as
\begin{align*}
    D_{f_k}(Q\lVert P) =  \int f_k(\frac{dP}{dQ}) dQ, 
\end{align*}
with $ f_k(t)\coloneqq\frac{t^k-k t+k-1}{k(k-1)}$.
Based on this family, we instantiate our ambiguity set in Equation~\ref{eq:sa_ambiguity_set} as $\mathcal{P}_{s,a} = \{P^{\prime}_{s,a}:\Delta_{S}\lvert D_{f_k}(P^{\prime}_{s,a}\lVert P_{s,a})\le \rho \}$ for some radius $\rho>0$. 
The Cressie-Read family of $f$-divergence includes $\chi^2$-divergence ($k=2$) and KL divergence ($k\rightarrow 1$).

One key challenge in developing DRRL algorithms using the formulation in Equation~\ref{eq:optimality} is that the expectation is taken over the ambiguity set $\gP$, which is computationally intensive even with the access to the center model $P$. 
Since we only have access to samples generated from the possibly misspecific model $P$, estimating the expectation with respect to other models $P'\in \gP$ is even more challenging. While importance sampling-based techniques can achieve this, the cost of high variance is still undesirable.
To solve this issue, we rely on the dual reformulation of Equation~\ref{eq:optimality}:
\begin{lemma}[\cite{duchi2021learning}]
    \label{lemma:cressie_dual}
    For any random variable $X\sim P$, define $\sigma_k(X, \eta) = -c_k(\rho) \sE_P[(\eta - X)_+^{k_*}]^{\frac{1}{k_*}} + \eta$ with $k_* = \frac{k}{k-1}$ and $c_k(\rho) = (1+k(k-1)\rho)^{\frac{1}{k}}$.
    Then
    \begin{align}
        \label{eq:dual_cressie}
        \inf_{Q \ll P}\{\sE_{Q}[X]: D_{f_k}(Q\lVert P)\le \rho\}&= \sup_{\eta\in \sR} \sigma_k(X, \eta),
    \end{align}
\end{lemma}

\vspace{-0.3cm}

Here $(x)_+ = \max\{x,0\}$. Equation~\ref{eq:dual_cressie} shows that protecting against the distribution shift is equivalent to optimizing the tail-performance of a model, as only the value below the dual variable $\eta$ are taken into account.
Another key insight from the reformulation is that as the growth of $f_k(t)$ for large $t$ becomes steeper for larger $k$, the $f$-divergence ball shrinks and the risk measure becomes less conservative. 
This bridges the gap between difference divergences, whereas previous literature, including \citet{yang2021theoretical} and \citet{zhou2021finitesample}, treats different divergences as separate.
By applying the dual reformulation, we can rewrite the Cressie-Read Bellman operator in Equation~\ref{eq:optimality} as
\begin{align}
    \label{eq:bellman_cressie}
    \gT_k(Q)(s,a) = r(s,a) + \gamma \sup_{\eta\in \sR}\sigma_k(\max_{a'\in \gA}Q(\cdot, a'), \eta).
\end{align}

\vspace{-0.4cm}

\subsection{Bias in Plug-in Estimator in Single Trajectory Setting}
\label{subsec:bias}
In this subsection, we aim to solve Equation~\ref{eq:bellman_cressie} using single-trajectory data, which has not been addressed by previous DRRL literature. 
As we can only observe one newly arrival sample each time, to design a online model-free DRRL algorithm, we need to approximate the expectation in Equation~\ref{eq:bellman_cressie} using that single sample properly.
As mentioned in Section~\ref{subsec:q-learning}, the design of the $Q$-learning algorithm \textbf{relies on an one-sample unbiased estimator of the true Bellman operator}. 
However, this convenience vanishes in the DR Bellman operator. 
To illustrate this, consider plugging only one sample into the Cressie-Read Bellman operator Equation~\ref{eq:bellman_cressie}:
\begin{align*}
 & r(s,a) + \gamma\sup_{\eta\in \sR}\{\eta -c_k(\rho)(\eta - \max_{a'} Q(s',a'))_+ \}\\
 &\quad = r(s,a) + \gamma \max_{a'} Q(s',a').
\end{align*}
This reduces to the non-robust Bellman operator and is obviously not an unbiased estimator for $\gT_k(Q)$. This example reveals the inherently more challenging nature of the online DRRL problem. Whereas non-robust RL only needs to improve the expectation of the cumulative return, improving the worst-case return requires more information about the system dynamics, which seems hopeless to be obtained from only one sample and  sharply contrasts with our target.

Even with the help of batch samples, deriving an appropriate estimator for the DR Bellman operator is still nontrivial. 
Consider a standard approach to construct estimators, sample average approximation (SAA): 
given a batch of sample size $n$ starting from a fix state-action pair $(s,a)$, i.e., $D_n = \{(s_i,a_i,s^{\prime}_i, r_i), i\in[n], (s_i, a_i)=(s,a)\}$, the SAA empirical Bellman operator is defined as:
\begin{align*}
    \widehat{\gT}_{k}(Q)(s,a, D_n) = r(s,a) + \gamma \sup_{\eta\in \sR}\widehat{\sigma}_k(\max_{a'\in \gA}Q(\cdot, a'), \eta, D_n).
\end{align*}
Here, $\widehat{\sigma}_k$ is the empirical Cressie-Read functional defined as 
\begin{align*}
    \widehat{\sigma}_k \coloneqq -c_k(\rho) [\sum_{i\in[n]}(\eta - \max_{a'\in \gA}Q(s_i', a'))_+^{k_*}/n]^{\frac{1}{k_*}} + \eta.
\end{align*}
As pointed out by \citet{liu2022distributionally}, the SAA estimator is biased, prompting the introduction of the multilevel Monte-Carlo method \citep{blanchet2015unbiased}. Specifically, it first obtains $N\in \sN^{+}$ samples from the distribution $\sP(N=n)= p_n = \epsilon(1-\epsilon)^n$, and then uses the simulator to draw $2^{N+1}$ samples $D_{2^{N+1}}$. The samples are further decomposed into two parts: $D_{:2^{N}}$ consists of the first $2^N$ samples, while $D_{2^{N}+1:}$ contains the remaining samples. Finally, the DR term in Equation~\ref{eq:bellman_cressie} is approximated by solving three optimization problems:
\begin{align*}
    \widehat{\gT}_k(Q)(s,a, D_n) &= r_1 + \max_{a'\in \gA} Q(s_1^{\prime}, a') + \frac{\Delta_{N, \delta}^q(Q)}{p_N},
\end{align*}
\begin{align*}
        \Delta_{N, \delta}^q(Q)& \coloneqq \sup_{\eta\ge 0} \widehat{\sigma}_k(\max_{a'\in \gA} Q(\cdot, a'), \eta, D_{2^{N+1}}) \\
        & \quad - \frac{1}{2}\sup_{\eta\ge 0} \widehat{\sigma}_k(\max_{a'\in \gA} Q(\cdot, a'), \eta, D_{:2^{N}})\\
        & \quad - \frac{1}{2}\sup_{\eta\ge 0} \widehat{\sigma}_k(\max_{a'\in \gA} Q(\cdot, a'), \eta, D_{2^{N}+1:}).
\end{align*}
However, this multilevel Monte-Carlo solution requires a large batch of samples for the same state-action pair before the next update, resulting in unbounded memory costs/computational time that are not practical. 
Furthermore, it is prohibited in the single-trajectory setting, where each step only one sample can be observed. 
Our experimental results show that simply approximating the Bellman operator with simulation data, without exploiting its structure, suffers from low data efficiency.
\subsection{Three-timescale Framework}
The $Q$-learning is solving the nonrobust Bellman operator's fixed point in a stochastic approximation manner.
A salient feature in the DR Bellman operator, compared with its nonrobust counterpart, is a bi-level optimization nature, i.e., jointly solving the dual parameter $\eta$ and the fixed point $Q$ of the Bellman optimality equation.
We revisit the stochastic approximation view of the $Q$-learning and develop a three-timescale framework, by a faster running estimate of the optimal dual parameter, and a slower update of the $Q$ table.

To solve Equation~\ref{eq:bellman_cressie} using a stochastic approximation template, we iteratively update the variables $\eta$ and $Q$ table as follows: for the $n$-th iteration after observing a new transition sample $(s_n, a_n, s_n', r_n)$ and some learning rates $\zeta_1, \zeta_2>0$,   
\begin{align*}
    \eta_{n+1} &= \eta_n - \zeta_1 * \text{Gradient of } \eta_n, \\ 
    Q_{n+1} &=  r_n + \zeta_2\gamma \sigma_k(\max_{a'\in \gA}Q_n(\cdot, a'), \eta_{n}).
\end{align*}
As the update of $\eta$ and $Q$ relies on each other, we keep the learning speeds of $\eta$ and $Q$, i.e., $\zeta_1$ and $\zeta_2$, different to stabilize the training process.
Additionally, due to the $(s,a)$-rectangular assumption, $\eta$ is independent across different $(s,a)$-pairs, while the $Q$ table depends on each other. 
The independent structure for $\eta$ allows it to be estimated more easily; so we approximate it in a faster loop, while for $Q$ we update it in a slower loop.

\subsection{Algorithmic Design}
In this subsection, we further instantiate the three-timescale framework to the Cressie-Read family of $f$-divergences.
First, we compute the gradient of $\sigma_{k}(\max_{a'\in \gA}Q(\cdot, a'), \eta)$ in Equation~\ref{eq:bellman_cressie} with respect to $\eta$.

\begin{algorithm}[htbp]
    \caption{Distributionally Robust $Q$-learning with Cressie-Read family of $f$-divergences}
    \label{alg:discount_dro_Q_chi}
    \begin{algorithmic}[1]
    \STATE {\bfseries Input:} Exploration rate $\epsilon$, Learning rates $\{\zeta_i(n)\}_{i\in[3]}$, Cressie-Read family parameter $k$, Ambiguity set radius $\rho$.
    \STATE {\bfseries Init:} Initialize $Q$, $Z$ and $\eta$ with zero.
    \FOR {$n = 1,2,\cdots$}
    \STATE Observe the state $s_n$, execute the action $a_n = \operatorname{\arg\max}_{a\in \mathcal{A}} Q(s_n, a)$ using $\epsilon$-greedy policy
    \STATE Observe the reward $r_n$ and next state $s_{n}^{\prime}$
    \STATE Update 
    \begin{align*}        
    Z_{1}(s_{n}, a_{n}) &\leftarrow  (1- \zeta_1(n))Z_{1}(s_{n}, a_{n})\\
    & + \zeta_1(n)(\eta(s_n, a_n)-\max_a Q(s_n^{\prime},a))_+^{k_*},\\
    Z_{2}(s_{n}, a_{n}) &\leftarrow (1- \zeta_1(n))Z_{2}(s_{n}, a_{n})\\
    & + \zeta_1(n) (\eta(s_n, a_n)-\max_a Q(s_n^{\prime},a))_+^{k_*-1}.
    \end{align*}
    \STATE Update \\
    {\small
    $
    \eta(s_n, a_n) \leftarrow \eta(s_n, a_n)
    + \zeta_2(n)(-c_k(\rho) Z_{1}^{\frac{1}{k_*}-1}(s_n, a_n)\cdot Z_{2}(s_n, a_n) + 1).
    $
    }
    \STATE 
    Update \\
    $
    Q(s_{n}, a_{n}) \leftarrow (1-\zeta_3(n)) Q(s_{n}, a_{n})
    + \zeta_3(n) (r_n - \gamma(c_k(\rho)Z_{1}^{\frac{1}{k_*}}(s_n, a_n) - \eta(s_n,a_n))).
    $
    \ENDFOR
    \end{algorithmic}
\end{algorithm}

\begin{lemma}[Sub-Gradient of the $\sigma_k$ dual function]

$$\partial \sigma_k (\max_{a' \in \mathcal{A}} Q(\cdot, a'), \eta) \in $$
\begin{equation}
 \label{eq:gradient_Cressie_Read} 
    \left \{ 
    \begin{aligned}
        &\{-c_k(\rho) Z_1^{\frac{1}{k^*}-1} \cdot Z_2+1\},\quad \eta > \max_{a{'} \in A} Q(\cdot, a'),\\
        &[-c_k(\rho) Z_1^{\frac{1}{k^*}-1} \cdot Z_2+1, 0],\quad \eta = \max_{a' \in \mathcal{A}} Q(\cdot, a'),\\
        &\{1\}, \quad \eta < \max_{a' \in \mathcal{A}} Q(\cdot, a'),
      \end{aligned}
      \right.
\end{equation}

where 
\begin{align}
 Z_1& =\mathbb{E}_{P}[(\eta-\max_{a'\in \gA}Q(\cdot, a'))_{+}^{k_*}],\\ 
 Z_2&=\mathbb{E}_P[(\eta-\max_{a'\in \gA}Q(\cdot, a'))_{+}^{k_*-1}].
\end{align}
\end{lemma}
Due to the nonlinearity in Equation~\ref{eq:gradient_Cressie_Read}, the plug-in gradient estimator is in fact biased.
The bias arises as for a random variable $X$, $\sE[f(X)] \neq f(\sE[X])$ for $f(x) = x^{\frac{1}{k_*}-1}$ in $Z_1^{\frac{1}{k_*-1}}$. 
To address this issue, we introduce another even faster timescale to estimate $Z_1$ and $Z_2$,
\begin{align}
\label{eq:Z1_chi}
Z_{1}(s_{n}, a_{n}) &\leftarrow (1- \zeta_1(n))Z_{1}(s_{n}, a_{n})\notag \\
    & \quad + \zeta_1(n)(\eta(s_n, a_n)-\max_{a'} Q(s_n^{\prime},a'))_+^{k_*}, \\
\label{eq:Z2_chi}
Z_{2}(s_{n}, a_{n}) &\leftarrow (1- \zeta_1(n))Z_{2}(s_{n}, a_{n})\notag \\
& \quad  + \zeta_1(n) (\eta(s_n, a_n)-\max_{a'} Q(s_n^{\prime},a'))_+^{k_*-1}. 
\end{align}
In the medium timescale, we approximate $\eta^{\star}(s,a)\coloneqq \arg\max_{\eta\in \gR}\sigma_k(\max_{a'\in\gA}Q(s,a'), \eta)$ by incrementally update the dual variable $\eta$ using the stochastic gradient descent method, where the true gradient computed in Equation~\ref{eq:gradient_Cressie_Read} is approximated by: 
\begin{align}
    \label{eq:eta_chi}
    & \eta(s_n, a_n) \leftarrow \eta(s_n, a_n)\notag \\
    & \quad + \zeta_2(n)( -c_k(\rho) Z_{1}^{\frac{1}{k_*}-1}(s_n,a_n)\cdot Z_{2}(s_n,a_n) + 1).
\end{align}
Finally, we update the DR $Q$ function in the slowest timescale using Equation~\ref{eq:chi_Q_discount}, 
 \begin{align}
  Q(s_{n}, a_{n}) \leftarrow & (1-\zeta_3(n)) Q(s_{n}, a_{n})\notag \\
 & + \zeta_3(n) \widehat{\gT}_{n,k}(Q)(s_n,a_n), \label{eq:chi_Q_discount}
 \end{align}
 where $\widehat{\gT}_{n,k} (Q)(s,a)$ is the empirical version of Equation~\ref{eq:bellman_cressie} in the $n$-th iteration:
 \begin{align*}
 &\widehat{\gT}_{n, k} (Q)(s_n,a_n) = r_n - \gamma(c_k(\rho)Z_{1}^{\frac{1}{k_*}}(s_n,a_n) - \eta(s_n,a_n)).
 \end{align*}
 Here $\zeta_1(n), \zeta_2(n)$ and $\zeta_3(n)$ are learning rates for three timescales at time $n$, which will be specified later.
We summarize the ingredients into our DR $Q$-learning (DRQ) algorithm (Algorithm~\ref{alg:discount_dro_Q_chi}), and prove the almost surely (a.s.) convergence of the algorithm as Theorem~\ref{thm:convergence_discount_dro_chi}. 
The proof is deferred in Appendix~\ref{subsec:proof_discount_chi}.
 \begin{theorem}
 \label{thm:convergence_discount_dro_chi}
 The estimators at the n-th step  in Algorithm~\ref{alg:discount_dro_Q_chi}, $(Z_{n,1}, Z_{n,2}, \eta_n, Q_n)$, converge to $(Z^{\star}_{1}, Z^{\star}_{2}, \eta^{\star}, Q^{\star})$ a.s. as $n\rightarrow \infty$, where $\eta^{\star}$ and $Q^{\star}$ are the fixed-point of the equation $Q = \mathcal{T}_k(Q)$, and $Z^{\star}_1$ and $Z^{\star}_2$ are the corresponding quantity under $\eta^{\star}$ and $Q^{\star}$.
 \end{theorem}

The proof establishes that, by appropriately selecting stepsizes to prioritize frequent updates of $Z_{n,1}$ and $Z_{n,2}$, followed by $\eta_n$, and with $Q_n$ updated at the slowest rate, the solution path of $(Z_{n,1}, Z_{n,2}, \eta_n, Q_n)$ closely tracks a system of three-dimensional ordinary differential equations (ODEs) considering martingale noise. 
Our approach is to generalize the classic machinery of two-timescale stochastic approximation~\citep{borkar2009stochastic} to a three-timescale framework, and use it to analyze our proposed algorithm.
See Appendix~\ref{sec:multi_timescale} for the detailed proof.
\section{Experiments}
We demonstrate the robustness and sample complexity of our DRQ algorithm in the Cliffwalking environment \citep{deletang2021model} and American put option environment (deferred in Appendix~\ref{sec:add_exp}). 
These environments provide a focused perspective on the policy and enable a clear understanding of the key parameters effects. 
We develop a deep learning version of DRQ and compare it with practical online and offline (robust) RL algorithms in classical control tasks, LunarLander and CartPole.
\subsection{Convergence and Sample Complexity}
Before we begin, let us outline the key findings and messages conveyed in this subsection:
\textbf{(1) Our ambiguity set design provides substantial robustness}, as demonstrated through comparisons with non-robust $Q$-learning and $R$-contamination ambiguity sets \citep{wang2021online}.
\textbf{(2) Our DRQ algorithm exhibits desirable sample complexity}, significantly outperforming the multi-level Monte Carlo based DRQ algorithm proposed by \citet{liu2022distributionally} and comparable to the sample complexity of the model-based DRRL algorithm by \citet{panaganti2022sample}.

\begin{figure}[htbp]
    \centering
    \subfloat[Environment]{
    \includegraphics[width=0.20\textwidth]{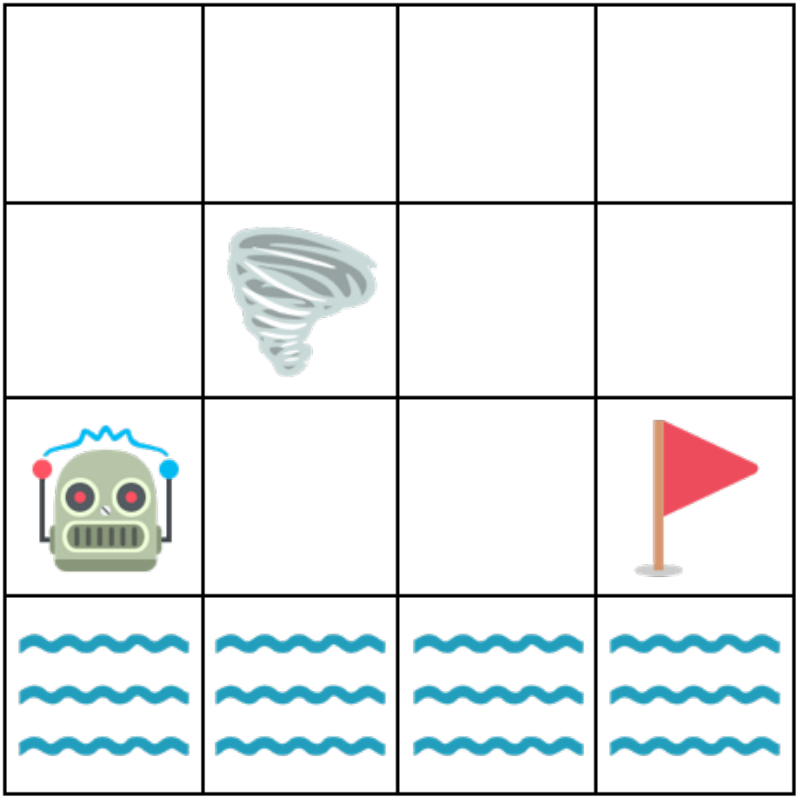}}
    \hfill 
    \subfloat[Nonrobust]{
    \includegraphics[width=0.20\textwidth]{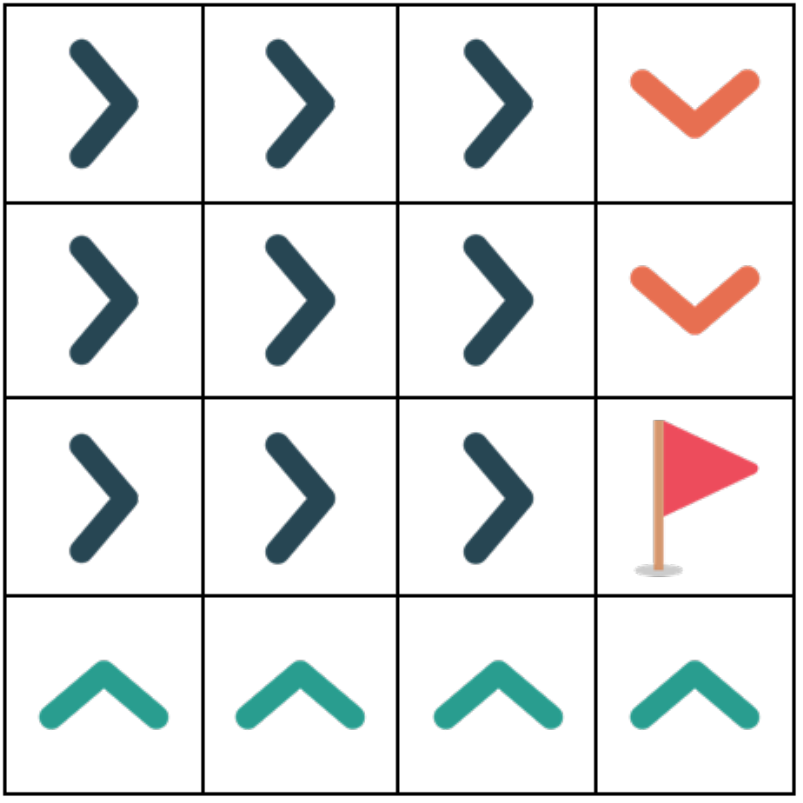}}
    \hfill 
    \subfloat[$\rho=1.0$]{
    \includegraphics[width=0.20\textwidth]{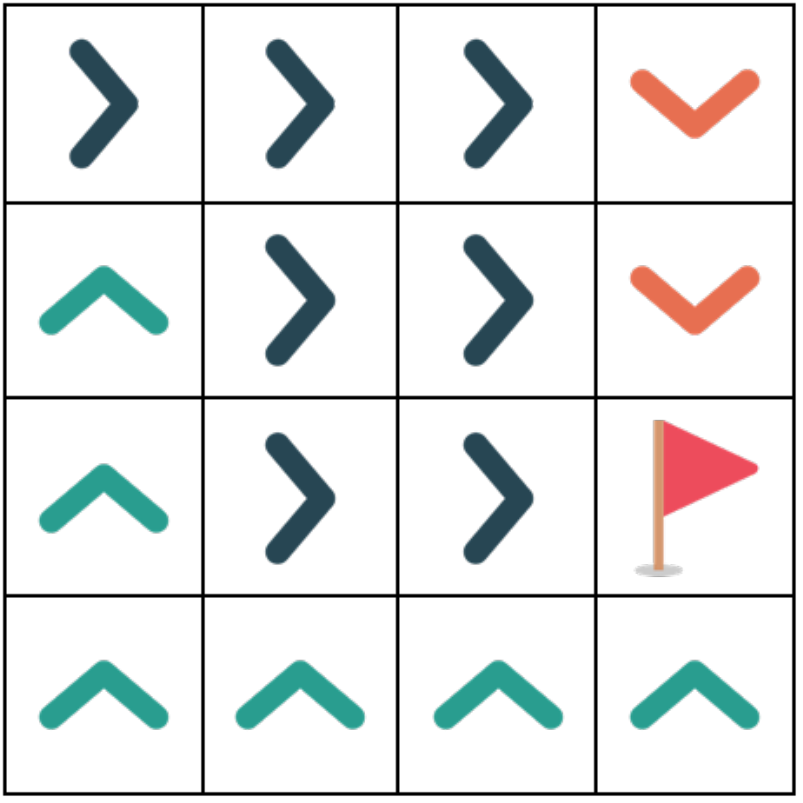}}
    \hfill 
    \subfloat[$\rho=1.5$]{
    \includegraphics[width=0.20\textwidth]{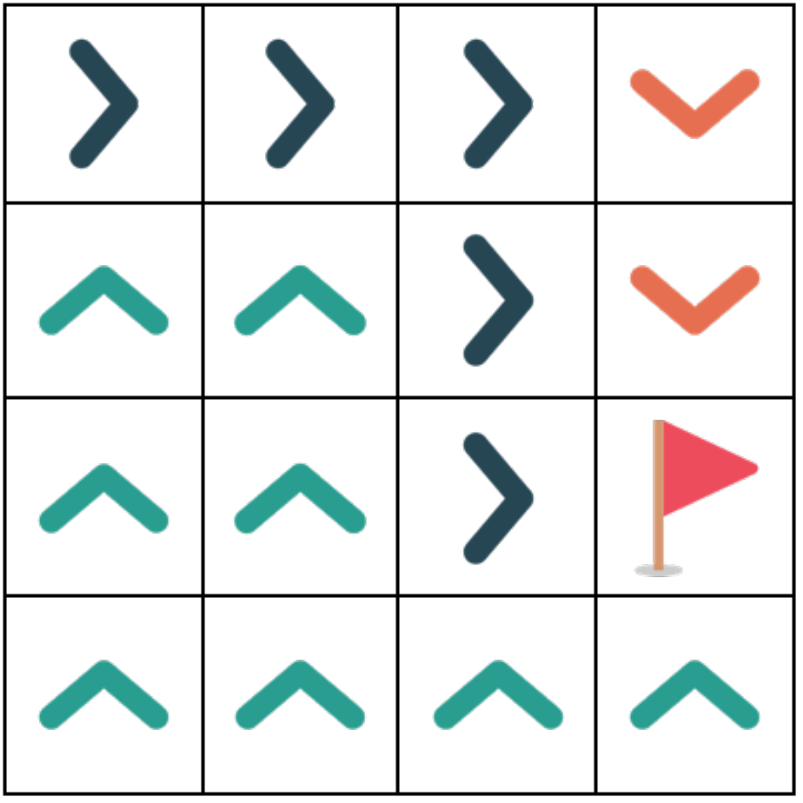}}
    \caption{The Cliffwalking environment and the learned policies for different $\rho$'s.}\label{fig:water_env}
\end{figure}

\textbf{Experiment Setup:} 
The Cliffwalking task is commonly used in risk-sensitive RL research \citep{deletang2021model}. Compared to the Frozen Lake environment used by \citet{panaganti2022sample}, Cliffwalking offers a more intuitive visualization of robust policies (see Figure~\ref{fig:water_env}). The task involves a robot navigating from an initial state of $(2,0)$ to a goal state of $(2,3)$. At each step, the robot is affected by wind, which causes it to move in a random direction with probability $p$. Reaching the goal state earns a reward of $+5$, while encountering a wave in the water region $\{(3,j) \mid 0 \leq j \leq 3 \}$ results in a penalty of $-1$. We train the agent in the nominal environment with $p=0.5$ for 3 million steps per run, using an $\epsilon$-greedy exploration strategy with $\epsilon=0.1$. We evaluate its performance in perturbed environments, varying the choices of $k$ and $\rho$ to demonstrate different levels of robustness. 
We set the stepsize parameters according to Assumption~\ref{as:stepsize}: $\zeta_1(t)=1 /(1 + (1 - \gamma)t^{0.6})$, $\zeta_2(t) = 1 /(1 + 0.1(1 - \gamma)t^{0.8})$, and $\zeta_3(t) = 1/(1 + 0.05(1 - \gamma) * t)$, where the discount factor is $\gamma=0.9$.
\begin{figure*}[htbp]
    \centering
    \subfloat[Return]{\label{fig:water_return}
    \includegraphics[height=3.1cm]{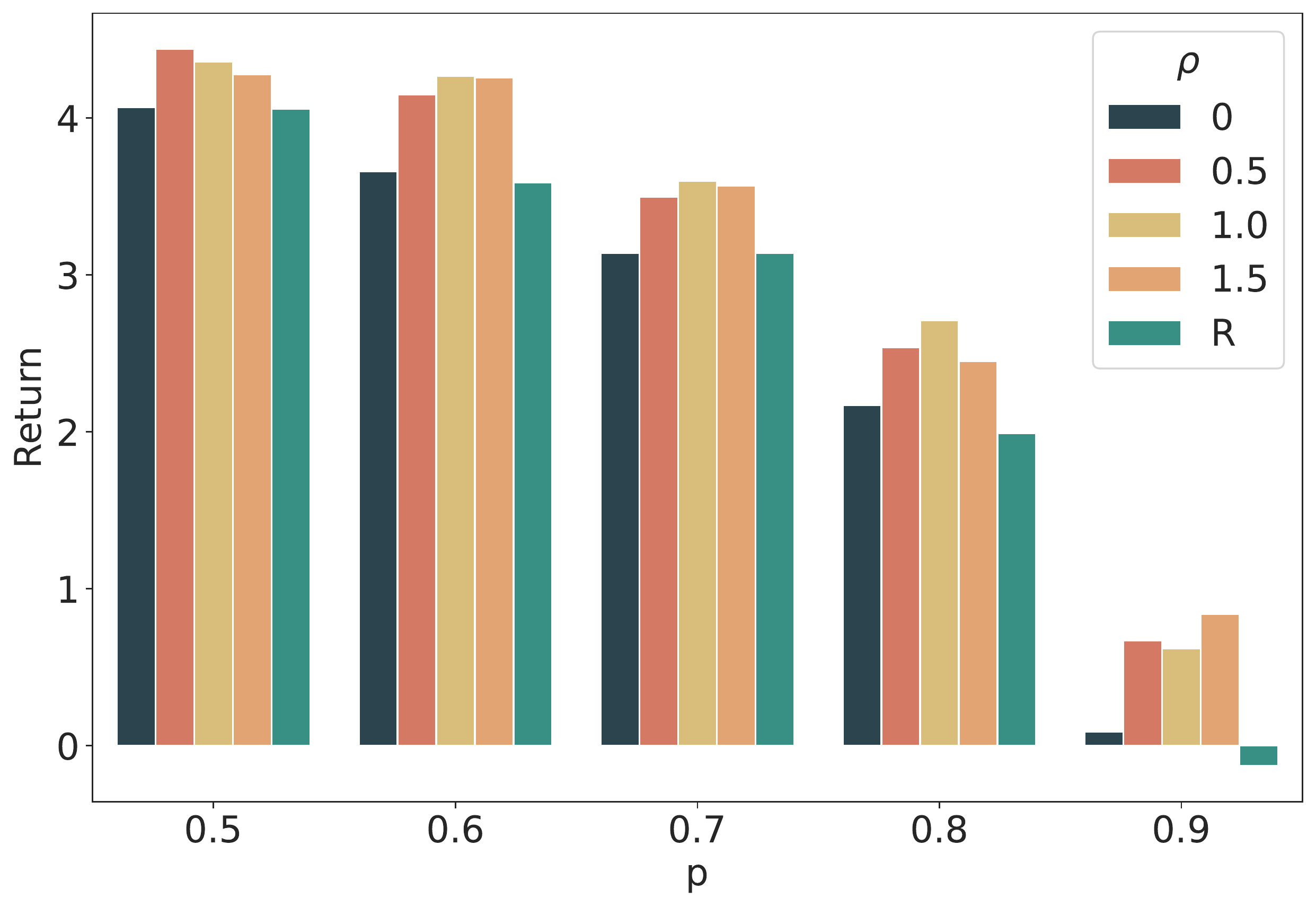}} 
    \subfloat[Episode length]{\label{fig:water_lengths}
    \includegraphics[height=3.1cm]{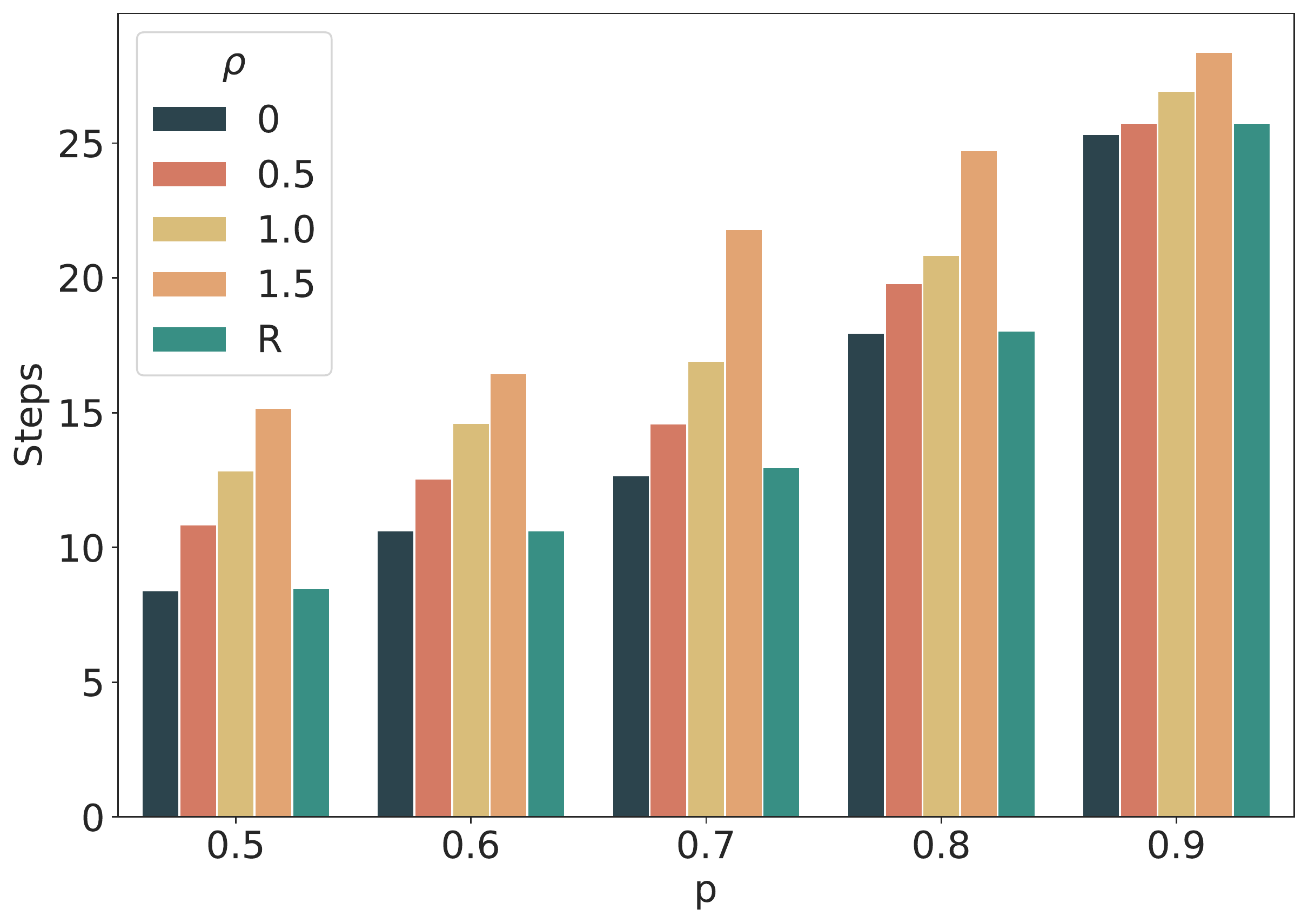}}
    \subfloat[Value of various $k$ and $\rho$]{
    \label{fig:kNrho}
    \includegraphics[height=3.1cm]{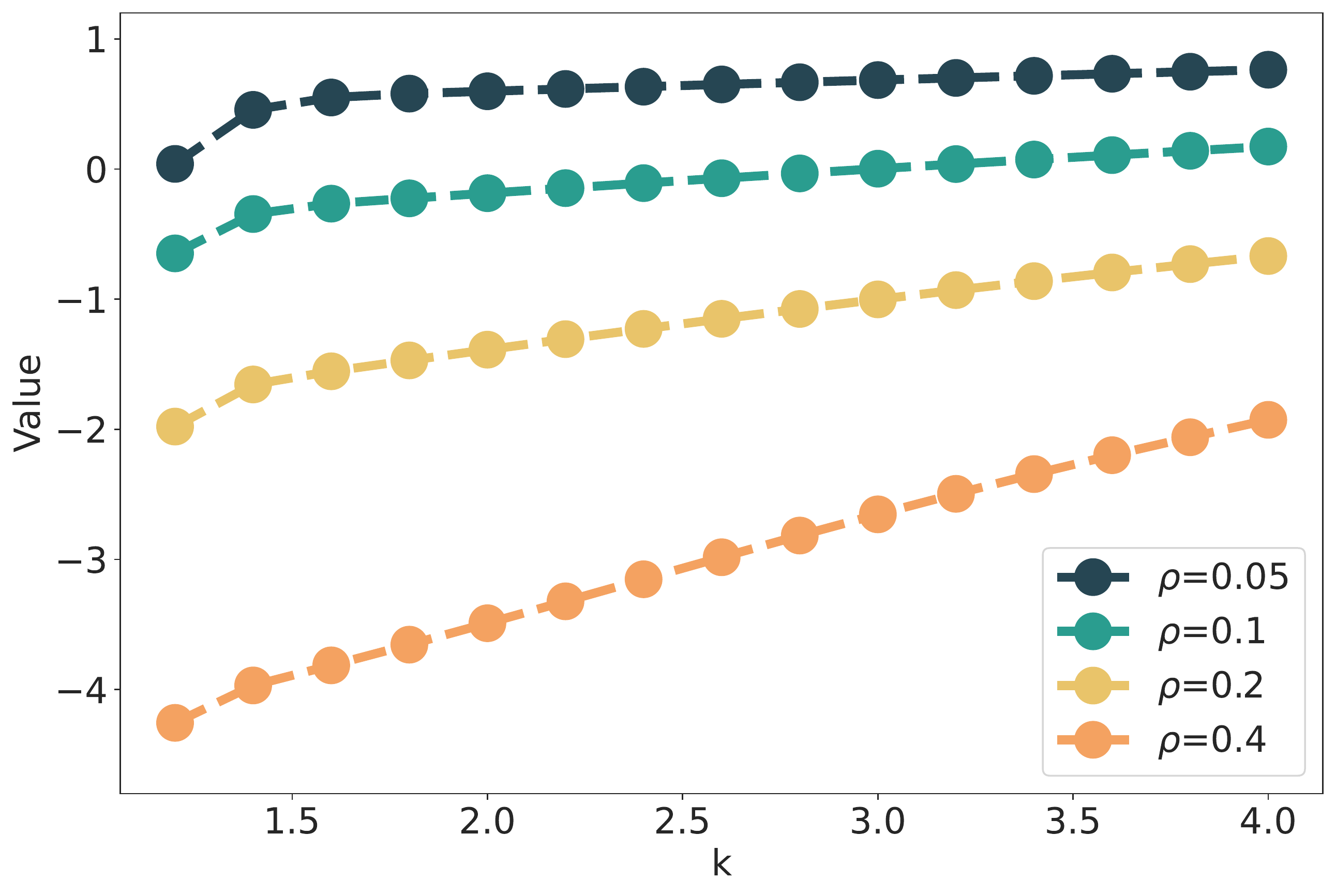}
    }
    \caption{Averaged return and steps with 100 random seeds in the perturbed environments. $\rho=0$ corresponds to the non-robust $Q$-learning. $R$ denotes the $R$-contamination ambiguity set.}
    \label{fig:water_result}
\end{figure*}

\begin{figure*}[htbp]
    \centering
    \includegraphics[width=.85\textwidth]{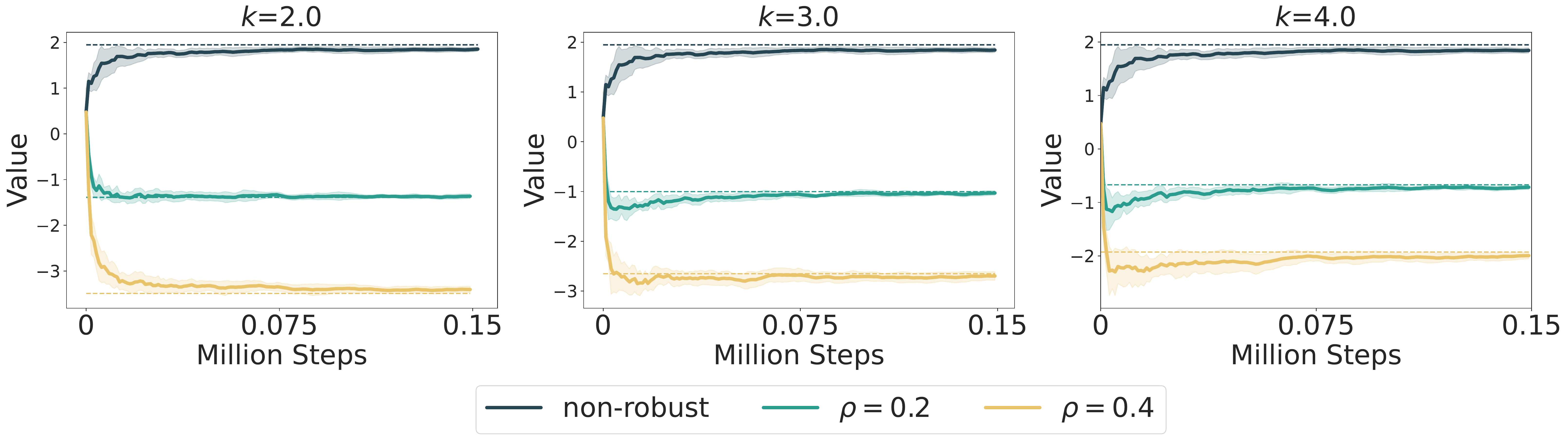}
    \caption{The training curves in the Cliffwalking environment. Each curve is averaged over 100 random seeds and shaded by their standard deviations. The dashed line is the optimal robust value with corresponding $k$ and $\rho$.}
    \label{fig:water_curves}
\end{figure*}

\textbf{Robustness:} 
To evaluate the robustness of the learned policies, we compare their cumulative returns in perturbed environments with $p \in \{0.5,0.6,0.7,0.8,0.9\}$ over 100 episodes per setting. 
We visulize the decision at each status in Figure~\ref{fig:water_env} with different robustness level $\rho$.
In particular, the more robust policy tends to avoid falling into the water, thus arrives to the goal state with a longer path by keeping going up before going right.
Figure~\ref{fig:water_return} shows the return distribution for each policy. Figure~\ref{fig:water_lengths} displays the time taken for the policies to reach the goal, and the more robust policy tends to spend more time, which quantitatively supports our observations in Figure~\ref{fig:water_env}. Interestingly, we find that the robust policies outperform the nonrobust one even in the nominal environment. 
For the different $\rho$'s, $\rho=1.0$ is the best within a relatively wide range ($p\in \{0.6,0.7,0.8\}$), while $\rho=1.5$ is preferred in the environment of extreme pertubation ($p=0.9$). 
This suggests that DRRL provides a elegant trade-off for different robustness preferences.

We also compare our model-free DRRL algorithm with the robust RL algorithm presented in \citet{wang2021online}, which also supports training using a single trajectory. 
The algorithm in \citet{wang2021online} uses an $R$-contamination ambiguity set. 
We select the best value of $R$ from $0.1$ to $0.9$ and other detailed descriptions in Appendix~\ref{sec:add_exp}. In most cases, the $R$-contamination based algorithm performs very similarly to the non-robust benchmark, and even performs worse in some cases (i.e., $p=0.8$ and $0.9$),  due to its excessive conservatism.
As we mentioned in Section~\ref{subsec:divergence}, larger $k$ would render the the risk measure less conservative and thus less sensitive to the change in the ball radius $\rho$, which is empirically confirmed by Figure~\ref{fig:kNrho}.

\begin{figure*}[htbp]
    \centering
    \includegraphics[width=0.85\textwidth]{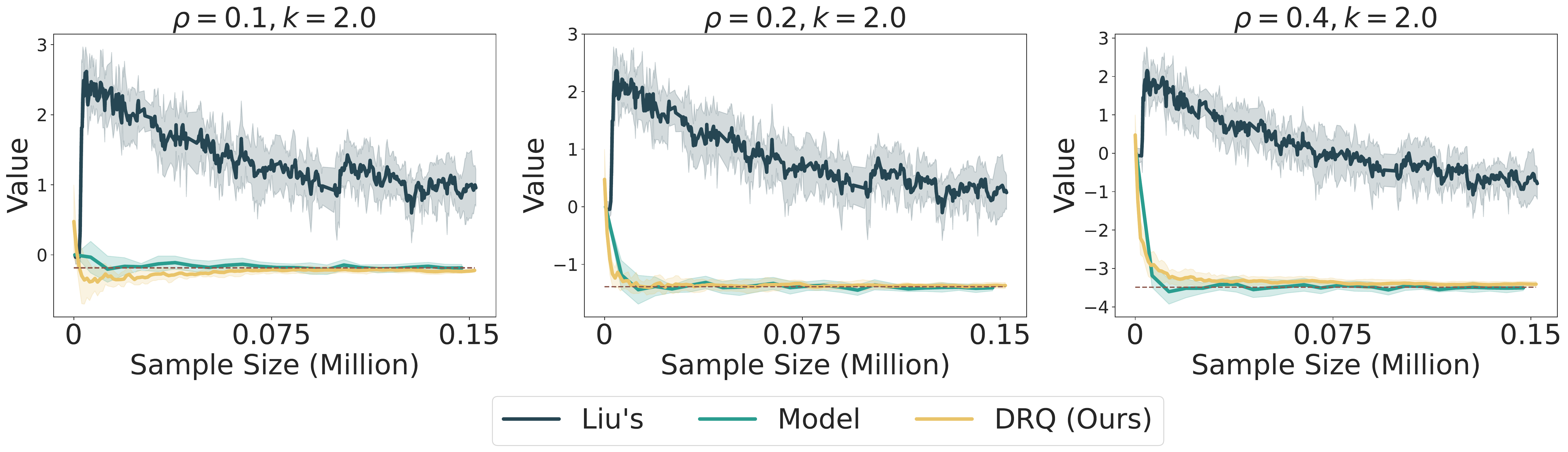}
    \caption{Sample complexity comparisons in Cliffwalking environment with Liu's and Model-based algorithms. Each curve is averaged over 100 random seeds and shaded by their standard deviations.}
    \label{fig:comparisons}
\end{figure*}

\textbf{Sample Complexity:}
The training curves in Figure~\ref{fig:water_curves} depict the estimated value $\max_a \widehat{Q}(s_0, a)$ (solid line) and the optimal robust value $V^*(s_0)$ (dashed line) for the initial state $s_0$. 
The results indicate that the estimated value converges quickly to the optimal value, regardless of the values of $k$ and $\rho$. Importantly, our DRQ algorithm achieves a similar convergence rate to the non-robust baseline (represented by the black line).
We further compare our algorithm with two robust baselines: the DRQ algorithm with a weak simulator proposed by  \citet{liu2022distributionally} (referred to as \emph{Liu's}), and the model-based algorithm introduced by \citet{panaganti2022sample} (referred to as \emph{Model}) in Figure~\ref{fig:comparisons}. 
To ensure a fair comparison, we set the same learning rate, $\zeta_3(t)$, for our DRQ algorithm and the $Q$-table update loop of the Liu's algorithm, as per their recommended choices. 

Our algorithm converges to the true DR value at a similar rate as the model-based algorithm, while the Liu's algorithm exhibits substantial deviation from the true value and converges relatively slowly. Our algorithm's superior sample efficiency is attributed to the utilization of first-order information to approximate optimal dual variables, whereas Liu's relies on a large amount of simulation data for an unbiased estimator.

\begin{figure*}[htbp]
    \centering
    \includegraphics[width=0.85\textwidth, height=6.5cm]{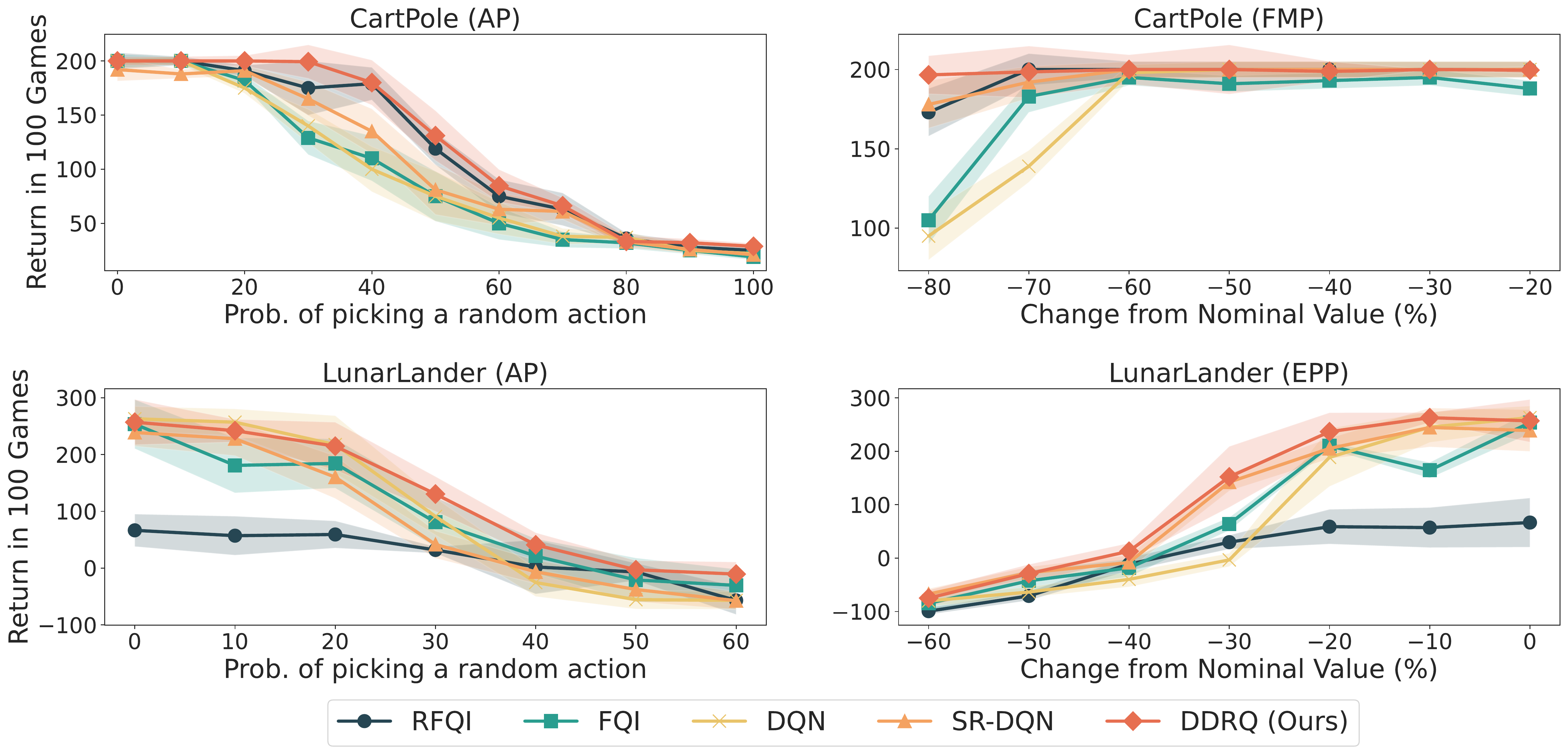}
    \caption{The return in the CartPole and LunarLander environment. Each curve is averaged over 100 random seeds and shaded by their standard deviations. AP: Action Perturbation; FMP: Force Mag Perturbation; EPP: Engines Power Perturbation.}
    \label{fig:pratical}
\end{figure*}

\subsection{Practical Implementation}
We validate the practicality of our DRQ framework by implementing a practical version, called the Deep Distributionally Robust $Q$-learning (\emph{DDRQ}) algorithm, based on the DQN algorithm \citep{mnih2015human}. We apply this algorithm to two classical control tasks from the OpenAI Gym \citep{brockman2016openai}: CartPole and LunarLander. 

Our practical algorithm, denoted as Algorithm~\ref{alg:discount_dro_Q_cressie_read_model}, is a variant of Algorithm~\ref{alg:discount_dro_Q_chi}.
Specifically, we adopt the Deep Q-Network (DQN) architecture~\citep{mnih2015human} and employ two sets of neural networks as functional approximators. One set, $Q_{\theta_1}$ and $Q_{\theta_2}$, serves as approximators for the $Q$ function, while the other set, $\eta_{\theta_3}$ and $\eta_{\theta_4}$, approximates the distributionally robust dual variable $\eta$. To enhance training stability, we introduce a target network, $Q_{\theta_2}$, for the fast $Q$ network $Q_{\theta_1}$ and $\eta_{\theta_4}$ for the fast dual variable $\eta$ network $\eta_{\theta_3}$.

Due to the approximation error introduced by neural networks and to further improve sample efficiency, our practical DDRQ algorithm adopts a two-timescale update approach. 
In this approach, our $Q$ network aims to minimize the Bellman error, while the dual variable $\eta$ network strives to maximize the DR $Q$ value defined in Equation~\ref{eq:bellman_cressie}. 
It's important to note that the two-timescale update approach could introduce bias in the convergence of the dual variable, and thus the dual variable $\eta$ may not the optimal dual variable for the primal problem. 
Given the primal-dual structure of this DR problem, this could render an even lower target value for the $Q$ network to learn. 
This approach can be understood as a robust update strategy for our original DRRL problem, share some spirits to the optimization techniques used in other algorithms like Variational Autoencoders (VAE)\citep{kingma2013auto}, Proximal Policy Optimization (PPO)\citep{schulman2017proximal}, and Maximum a Posteriori Policy Optimization (MPO)~\citep{abdolmaleki2018maximum}.  Additional experimental details can be found in Appendix \ref{subsec:practical_exp}.

To assess the effectiveness of our DDRQ algorithm, we compare it against the RFQI algorithm \citep{panaganti2022robust}, the soft-robust RL algorithm \citep{derman2018soft}, and the non-robust DQN and FQI algorithms. This comparison encompasses representative practical (robust) reinforcement learning algorithms for both online and offline datasets.
To evaluate the robustness of the learned policies, we introduce action and physical environment perturbations. For action perturbation, we simulate the perturbations by varying the probability $\epsilon$ of randomly selecting an action for both CartPole and LunarLander tasks. We test with $\epsilon\in \{0,0.1,0.2,\cdots,1.0\}$ for CartPole and $\epsilon\in \{0,0.1,0.2,\cdots,0.6\}$ for LunarLander.
Regarding physical environment perturbation in LunarLander, we decrease the power of all the main engine and side engines by the same proportions, ranging from 0 to $0.6$. For CartPole, we reduce the "force mag" parameter from $0.2$ to $0.8$.
We set the same ambiguity set radius for both our DDRQ and RFQI algorithm for fair comparisons.
Figure~\ref{fig:pratical} illustrates how our DDRQ algorithm successfully learns robust policies across all tested tasks, achieving comparable performance to other robust counterparts such as RFQI and SR-DQN. 
Conversely, the non-robust DQN and FQI algorithms fail to learn robust policies and deteriorate significantly even under slight perturbations.
It is worth noting that RFQI does not perform well in the LunarLander environment, despite using the official code provided by the authors. This outcome could be attributed to the restriction to their TV distance in constructing the ambiguity set, while our Creass-Read ambiguity set can be flexibily chosen to well adopted to the environment nature. 
Additionally, the soft-robust RL algorithm requires generating data based on multiple models within the ambiguity set. This process can be excessively time-consuming, particularly in large-scale applications.
\section{Conclusion}
In this paper, we introduce our DRQ algorithm, a fully model-free DRRL algorithm trained on a single trajectory.
By leveraging the stochastic approximation framework, we effectively tackle the joint optimization problem involving the state-action function and the DR dual variable. 
Through an extension of the classic two-timescale stochastic approximation framework, we establish the asymptotic convergence of our algorithm to the optimal DR policy. Our extensive experimentation showcases the convergence, sample efficiency, and robustness improvements achieved by our approach, surpassing non-robust methods and other robust RL algorithms.
Our DDRQ algorithm further validates the practicality of our algorithmic framework.

\section*{Acknowledgements}
This work is generously supported by the General Research Fund [Grants 16208120, and 16214121] from the Hong Kong Research Grants Council, the NSF grants CCF-2312205 and CCF-2312204.

\section*{Impact Statement}
This paper presents work whose goal is to advance the ﬁeld of Machine Learning, especially to enhance the robustness of the widely-used reinforcement learning algorithms. There are many potential societal consequences of our work, none which we feel must be speciﬁcally highlighted here.

\bibliography{reference}

\begin{thebibliography}{38}
\providecommand{\natexlab}[1]{#1}
\providecommand{\url}[1]{\texttt{#1}}
\expandafter\ifx\csname urlstyle\endcsname\relax
  \providecommand{\doi}[1]{doi: #1}\else
  \providecommand{\doi}{doi: \begingroup \urlstyle{rm}\Url}\fi

\bibitem[Abdolmaleki et~al.(2018)Abdolmaleki, Springenberg, Tassa, Munos,
  Heess, and Riedmiller]{abdolmaleki2018maximum}
Abdolmaleki, A., Springenberg, J.~T., Tassa, Y., Munos, R., Heess, N., and
  Riedmiller, M.
\newblock Maximum a posteriori policy optimisation.
\newblock \emph{arXiv preprint arXiv:1806.06920}, 2018.

\bibitem[Abdullah et~al.(2019)Abdullah, Ren, Ammar, Milenkovic, Luo, Zhang, and
  Wang]{abdullah2019wasserstein}
Abdullah, M.~A., Ren, H., Ammar, H.~B., Milenkovic, V., Luo, R., Zhang, M., and
  Wang, J.
\newblock Wasserstein {{Robust Reinforcement Learning}}, 2019.
\newblock URL \url{http://arxiv.org/abs/1907.13196}.

\bibitem[Badrinath \& Kalathil(2021)Badrinath and
  Kalathil]{badrinath2021robust}
Badrinath, K.~P. and Kalathil, D.
\newblock Robust reinforcement learning using least squares policy iteration
  with provable performance guarantees.
\newblock In \emph{International Conference on Machine Learning}, pp.\
  511--520. PMLR, 2021.

\bibitem[Blanchet \& Glynn(2015)Blanchet and Glynn]{blanchet2015unbiased}
Blanchet, J.~H. and Glynn, P.~W.
\newblock Unbiased monte carlo for optimization and functions of expectations
  via multi-level randomization.
\newblock In \emph{2015 Winter Simulation Conference (WSC)}, pp.\  3656--3667.
  IEEE, 2015.

\bibitem[Borkar(2009)]{borkar2009stochastic}
Borkar, V.~S.
\newblock \emph{Stochastic approximation: a dynamical systems viewpoint},
  volume~48.
\newblock Springer, 2009.

\bibitem[Borkar \& Meyn(2000)Borkar and Meyn]{borkar2000ode}
Borkar, V.~S. and Meyn, S.~P.
\newblock The ode method for convergence of stochastic approximation and
  reinforcement learning.
\newblock \emph{SIAM Journal on Control and Optimization}, 38\penalty0
  (2):\penalty0 447--469, 2000.

\bibitem[Brockman et~al.(2016)Brockman, Cheung, Pettersson, Schneider,
  Schulman, Tang, and Zaremba]{brockman2016openai}
Brockman, G., Cheung, V., Pettersson, L., Schneider, J., Schulman, J., Tang,
  J., and Zaremba, W.
\newblock Openai gym.
\newblock \emph{arXiv preprint arXiv:1606.01540}, 2016.

\bibitem[Cox et~al.(1979)Cox, Ross, and Rubinstein]{cox1979option}
Cox, J.~C., Ross, S.~A., and Rubinstein, M.
\newblock Option pricing: A simplified approach.
\newblock \emph{Journal of financial Economics}, 7\penalty0 (3):\penalty0
  229--263, 1979.

\bibitem[Cressie \& Read(1984)Cressie and Read]{cressie1984multinomial}
Cressie, N. and Read, T.~R.
\newblock Multinomial goodness-of-fit tests.
\newblock \emph{Journal of the Royal Statistical Society: Series B
  (Methodological)}, 46\penalty0 (3):\penalty0 440--464, 1984.

\bibitem[Del{\'e}tang et~al.(2021)Del{\'e}tang, Grau-Moya, Kunesch, Genewein,
  Brekelmans, Legg, and Ortega]{deletang2021model}
Del{\'e}tang, G., Grau-Moya, J., Kunesch, M., Genewein, T., Brekelmans, R.,
  Legg, S., and Ortega, P.~A.
\newblock Model-free risk-sensitive reinforcement learning.
\newblock \emph{arXiv preprint arXiv:2111.02907}, 2021.

\bibitem[Derman et~al.(2018)Derman, Mankowitz, Mann, and
  Mannor]{derman2018soft}
Derman, E., Mankowitz, D.~J., Mann, T.~A., and Mannor, S.
\newblock Soft-robust actor-critic policy-gradient.
\newblock \emph{arXiv preprint arXiv:1803.04848}, 2018.

\bibitem[Duchi \& Namkoong(2021)Duchi and Namkoong]{duchi2021learning}
Duchi, J.~C. and Namkoong, H.
\newblock Learning models with uniform performance via distributionally robust
  optimization.
\newblock \emph{The Annals of Statistics}, 49\penalty0 (3):\penalty0
  1378--1406, 2021.

\bibitem[Goyal \& Grand-Clement(2022)Goyal and Grand-Clement]{goyal2022robust}
Goyal, V. and Grand-Clement, J.
\newblock Robust markov decision processes: Beyond rectangularity.
\newblock \emph{Mathematics of Operations Research}, 2022.

\bibitem[Ho et~al.(2021)Ho, Petrik, and Wiesemann]{ho2021partial}
Ho, C.~P., Petrik, M., and Wiesemann, W.
\newblock Partial policy iteration for l1-robust markov decision processes.
\newblock \emph{J. Mach. Learn. Res.}, 22:\penalty0 275--1, 2021.

\bibitem[Iyengar(2005)]{iyengar2005robust}
Iyengar, G.~N.
\newblock Robust dynamic programming.
\newblock \emph{Mathematics of Operations Research}, 30\penalty0 (2):\penalty0
  257--280, 2005.

\bibitem[Kingma \& Welling(2013)Kingma and Welling]{kingma2013auto}
Kingma, D.~P. and Welling, M.
\newblock Auto-encoding variational bayes.
\newblock \emph{arXiv preprint arXiv:1312.6114}, 2013.

\bibitem[Lim et~al.(2013)Lim, Xu, and Mannor]{lim2013reinforcement}
Lim, S.~H., Xu, H., and Mannor, S.
\newblock Reinforcement learning in robust markov decision processes.
\newblock \emph{Advances in Neural Information Processing Systems}, 26, 2013.

\bibitem[Liu et~al.(2022)Liu, Bai, Blanchet, Dong, Xu, Zhou, and
  Zhou]{liu2022distributionally}
Liu, Z., Bai, Q., Blanchet, J., Dong, P., Xu, W., Zhou, Z., and Zhou, Z.
\newblock Distributionally robust $ q $-learning.
\newblock In \emph{International Conference on Machine Learning}, pp.\
  13623--13643. PMLR, 2022.

\bibitem[Ma et~al.(2022)Ma, Liang, Xia, Zhang, Blanchet, Liu, Zhao, and
  Zhou]{ma2022distributionally}
Ma, X., Liang, Z., Xia, L., Zhang, J., Blanchet, J., Liu, M., Zhao, Q., and
  Zhou, Z.
\newblock Distributionally robust offline reinforcement learning with linear
  function approximation.
\newblock \emph{arXiv preprint arXiv:2209.06620}, 2022.

\bibitem[Mannor et~al.(2004)Mannor, Simester, Sun, and
  Tsitsiklis]{mannor2004bias}
Mannor, S., Simester, D., Sun, P., and Tsitsiklis, J.~N.
\newblock Bias and variance in value function estimation.
\newblock In \emph{Proceedings of the twenty-first international conference on
  Machine learning}, pp.\ ~72, 2004.

\bibitem[Mnih et~al.(2015)Mnih, Kavukcuoglu, Silver, Rusu, Veness, Bellemare,
  Graves, Riedmiller, Fidjeland, Ostrovski, et~al.]{mnih2015human}
Mnih, V., Kavukcuoglu, K., Silver, D., Rusu, A.~A., Veness, J., Bellemare,
  M.~G., Graves, A., Riedmiller, M., Fidjeland, A.~K., Ostrovski, G., et~al.
\newblock Human-level control through deep reinforcement learning.
\newblock \emph{nature}, 518\penalty0 (7540):\penalty0 529--533, 2015.

\bibitem[Neufeld \& Sester(2022)Neufeld and Sester]{Neufeld2022RobustQA}
Neufeld, A. and Sester, J.
\newblock Robust q-learning algorithm for markov decision processes under
  wasserstein uncertainty.
\newblock \emph{ArXiv}, abs/2210.00898, 2022.

\bibitem[Nilim \& El~Ghaoui(2005)Nilim and El~Ghaoui]{nilim2005robust}
Nilim, A. and El~Ghaoui, L.
\newblock Robust control of markov decision processes with uncertain transition
  matrices.
\newblock \emph{Operations Research}, 53\penalty0 (5):\penalty0 780--798, 2005.

\bibitem[Panaganti \& Kalathil(2022)Panaganti and
  Kalathil]{panaganti2022sample}
Panaganti, K. and Kalathil, D.
\newblock Sample complexity of robust reinforcement learning with a generative
  model.
\newblock In \emph{International Conference on Artificial Intelligence and
  Statistics}, pp.\  9582--9602. PMLR, 2022.

\bibitem[Panaganti et~al.(2022)Panaganti, Xu, Kalathil, and
  Ghavamzadeh]{panaganti2022robust}
Panaganti, K., Xu, Z., Kalathil, D., and Ghavamzadeh, M.
\newblock Robust reinforcement learning using offline data.
\newblock \emph{Advances in neural information processing systems},
  35:\penalty0 32211--32224, 2022.

\bibitem[Roy et~al.(2017)Roy, Xu, and Pokutta]{roy2017reinforcement}
Roy, A., Xu, H., and Pokutta, S.
\newblock Reinforcement learning under model mismatch.
\newblock \emph{Advances in neural information processing systems}, 30, 2017.

\bibitem[Schulman et~al.(2017)Schulman, Wolski, Dhariwal, Radford, and
  Klimov]{schulman2017proximal}
Schulman, J., Wolski, F., Dhariwal, P., Radford, A., and Klimov, O.
\newblock Proximal policy optimization algorithms.
\newblock \emph{arXiv preprint arXiv:1707.06347}, 2017.

\bibitem[Shi \& Chi()Shi and Chi]{shi2022distributionally}
Shi, L. and Chi, Y.
\newblock Distributionally {{Robust Model-Based Offline Reinforcement
  Learning}} with {{Near-Optimal Sample Complexity}}.
\newblock URL \url{http://arxiv.org/abs/2208.05767}.

\bibitem[Silver et~al.(2016)Silver, Huang, Maddison, Guez, Sifre, Van
  Den~Driessche, Schrittwieser, Antonoglou, Panneershelvam, Lanctot,
  et~al.]{silver2016mastering}
Silver, D., Huang, A., Maddison, C.~J., Guez, A., Sifre, L., Van Den~Driessche,
  G., Schrittwieser, J., Antonoglou, I., Panneershelvam, V., Lanctot, M.,
  et~al.
\newblock Mastering the game of go with deep neural networks and tree search.
\newblock \emph{nature}, 529\penalty0 (7587):\penalty0 484--489, 2016.

\bibitem[Tamar et~al.(2014)Tamar, Mannor, and Xu]{tamar2013scaling}
Tamar, A., Mannor, S., and Xu, H.
\newblock Scaling up robust mdps using function approximation.
\newblock In \emph{International conference on machine learning}, pp.\
  181--189. PMLR, 2014.

\bibitem[Tsitsiklis(1994)]{tsitsiklis1994asynchronous}
Tsitsiklis, J.~N.
\newblock Asynchronous stochastic approximation and q-learning.
\newblock \emph{Machine learning}, 16:\penalty0 185--202, 1994.

\bibitem[Vinyals et~al.(2019)Vinyals, Babuschkin, Czarnecki, Mathieu, Dudzik,
  Chung, Choi, Powell, Ewalds, Georgiev, Oh, Horgan, Kroiss, Danihelka, Huang,
  Sifre, Cai, Agapiou, Jaderberg, Vezhnevets, Leblond, Pohlen, Dalibard,
  Budden, Sulsky, Molloy, Paine, Gulcehre, Wang, Pfaff, Wu, Ring, Yogatama,
  Wünsch, McKinney, Smith, Schaul, Lillicrap, Kavukcuoglu, Hassabis, Apps, and
  Silver]{vinyals2019grandmaster}
Vinyals, O., Babuschkin, I., Czarnecki, W.~M., Mathieu, M., Dudzik, A., Chung,
  J., Choi, D.~H., Powell, R., Ewalds, T., Georgiev, P., Oh, J., Horgan, D.,
  Kroiss, M., Danihelka, I., Huang, A., Sifre, L., Cai, T., Agapiou, J.~P.,
  Jaderberg, M., Vezhnevets, A.~S., Leblond, R., Pohlen, T., Dalibard, V.,
  Budden, D., Sulsky, Y., Molloy, J., Paine, T.~L., Gulcehre, C., Wang, Z.,
  Pfaff, T., Wu, Y., Ring, R., Yogatama, D., Wünsch, D., McKinney, K., Smith,
  O., Schaul, T., Lillicrap, T., Kavukcuoglu, K., Hassabis, D., Apps, C., and
  Silver, D.
\newblock Grandmaster level in {{StarCraft II}} using multi-agent reinforcement
  learning.
\newblock 575\penalty0 (7782):\penalty0 350--354, 2019.
\newblock \doi{10.1038/s41586-019-1724-z}.

\bibitem[Wang \& Zou(2021)Wang and Zou]{wang2021online}
Wang, Y. and Zou, S.
\newblock Online robust reinforcement learning with model uncertainty.
\newblock \emph{Advances in Neural Information Processing Systems},
  34:\penalty0 7193--7206, 2021.

\bibitem[Wiesemann et~al.(2013)Wiesemann, Kuhn, and
  Rustem]{wiesemann2013robust}
Wiesemann, W., Kuhn, D., and Rustem, B.
\newblock Robust markov decision processes.
\newblock \emph{Mathematics of Operations Research}, 38\penalty0 (1):\penalty0
  153--183, 2013.

\bibitem[Xu \& Mannor(2010)Xu and Mannor]{xu2010distributionally}
Xu, H. and Mannor, S.
\newblock Distributionally robust markov decision processes.
\newblock \emph{Advances in Neural Information Processing Systems}, 23, 2010.

\bibitem[Yang(2018)]{Yang2018WassersteinDR}
Yang, I.
\newblock Wasserstein distributionally robust stochastic control: A data-driven
  approach.
\newblock \emph{IEEE Transactions on Automatic Control}, 66:\penalty0
  3863--3870, 2018.

\bibitem[Yang et~al.(2022)Yang, Zhang, and Zhang]{yang2021theoretical}
Yang, W., Zhang, L., and Zhang, Z.
\newblock Toward theoretical understandings of robust markov decision
  processes: Sample complexity and asymptotics.
\newblock \emph{The Annals of Statistics}, 50\penalty0 (6):\penalty0
  3223--3248, 2022.

\bibitem[Zhou et~al.(2021)Zhou, Zhou, Bai, Qiu, Blanchet, and
  Glynn]{zhou2021finitesample}
Zhou, Z., Zhou, Z., Bai, Q., Qiu, L., Blanchet, J., and Glynn, P.
\newblock Finite-sample regret bound for distributionally robust offline
  tabular reinforcement learning.
\newblock In \emph{International Conference on Artificial Intelligence and
  Statistics}, pp.\  3331--3339. PMLR, 2021.

\end{thebibliography}
\bibliographystyle{icml2024}

\newpage
\appendix
\onecolumn
\section*{Appendix}

In the subsequent sections, we delve into the experimental specifics and provide the technical proofs that were not included in the primary content.

In Section~\ref{sec:add_exp}, we commence by showcasing an additional experiment on the American call option. This aligns with the convergence and sample complexity discussions from the main content. We then elucidate the intricacies of Liu's algorithm to facilitate a transparent comparison with our methodology. Lastly, we discuss the algorithmic intricacies of our DDRQ algorithm and provide details on the experiments that were previously omitted.

In Section~\ref{sec:multi_timescale}, to prove Theorem~\ref{thm:convergence_discount_dro_chi}, we begin by extending the two-timescale stochastic approximation framework to a three-timescale one. Following this, we adapt it to our algorithm, ensuring all requisite conditions are met.

\section{Additional Experiments Details}
\label{sec:add_exp}
\subsection{Experiment on the American Put Option Problem}
In this section, we present additional experimental results from a simulated American put option problem~\citep{cox1979option} that has been previously studied in robust RL literature ~\citep{zhou2021finitesample,tamar2013scaling}.
The problem involves holding a put option in multiple stages, whose payoff depends on the price of a financial asset that follows a Bernoulli distribution. 
Specifically, the next price $s_{h+1}$ at stage $h+1$ follows, 
\begin{equation}
s_{h+1}= \begin{cases}c_{u} s_{h}, & \text { w.p. } p_0, \\ c_{d} s_{h}, & \text { w.p. } 1-p_0,\end{cases}
\end{equation}
where the $c_u$ and $c_d$ are the price up and down factors and $p_0$ is the probability that the price goes up. The initial price $s_0$ is uniformly sampled from $[\kappa - \epsilon, \kappa + \epsilon]$, where $\kappa=100$ is the strike price and $\epsilon=5$ in our simulation. The agent can take an action to exercise the option ($a_h=1$) or not exercise ($a_h=0$) at the time step $h$. If exercising the option, the agent receives a reward $\max(0, \kappa-s_h)$ and the state transits into an exit state.
Otherwise, the price will fluctuate based on the above model and no reward will be assigned. 
Moreover we introduce a discount structure in this problem, i.e., the $1$ reward in the stage $h+1$ worths $\gamma$ in stage $h$ as our algorithm is designed for discounted RL setting. 
In our experiments, we set $H=5$, $c_u=1.02$, $c_d=0.98$ and $\gamma = 0.95$. We limit the price in $[80, 140]$ and discretize with the precision of 1 decimal place. Thus the state space size $|\mathcal{S}|=602$.

\begin{figure}[htbp]
    \centering
    \includegraphics[width=0.4\textwidth]{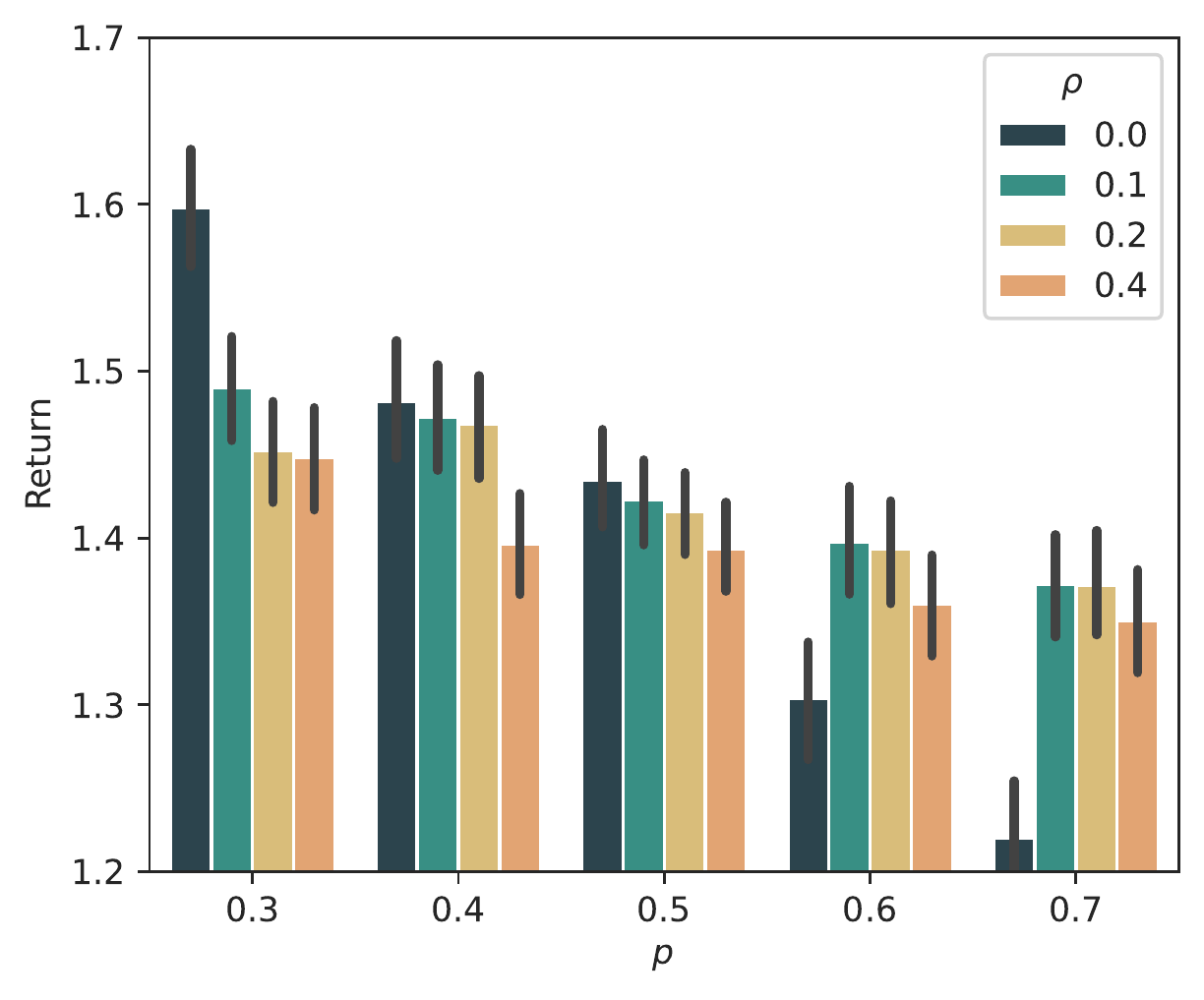}
    \caption{Averaged return in the American call option problem. $\rho=0.0$ is the non-robust $Q$-learning.}
    \label{fig:option_return}
\end{figure}

We first demonstrate the robustness gain of our DR $Q$-learning algorithm by comparing with the non-robust $Q$-learning algorithm, and investigate the effect of different robustness levels by varying $\rho$. 
Each agent is trained for $10^7$ steps with an $\epsilon$-greedy exploration policy of $\epsilon=0.2$ and evaluated in perturbed environments. 
We use the same learning rates for the three timescales in our DR $Q$-learning algorithm as in the Cliffwalking environment: $\zeta_1(t) = 1/(1+(1-\gamma) t^{0.6})$, $\zeta_2(t) = 1/(1+0.1*(1-\gamma) t^{0.8})$, and $\zeta_3(t) = 1/(1+0.01*(1-\gamma) t)$.
For the non-robust $Q$-learning we set the same learning rate as in our $Q$-update, i.e., $\zeta_3(t)$.
We perturb the transition probability to the price up and down status $p=\{0.3,0.4, 0.5, 0.6, 0.7\}$, and evaluate each agent for $5000$ episodes.
Figure~\ref{fig:option_return} reports the average return and one standard deviation level.  
The non-robust $Q$-learning performs best when the price tends to decrease and the market gets more benefitial ($p=\{0.3, 0.4, 0.5\}$), which benefits the return of holding an American put option. 
However, when the prices tend to increase and the market is riskier ($p=\{0.6, 0.7\}$), our DR $Q$-learning algorithm significantly outperforms the non-robust counterpart, demonstrating the robustness gain of our algorithm against worst-case scenarios.

\begin{figure}[htbp]
    \centering
    \includegraphics[width=\columnwidth]{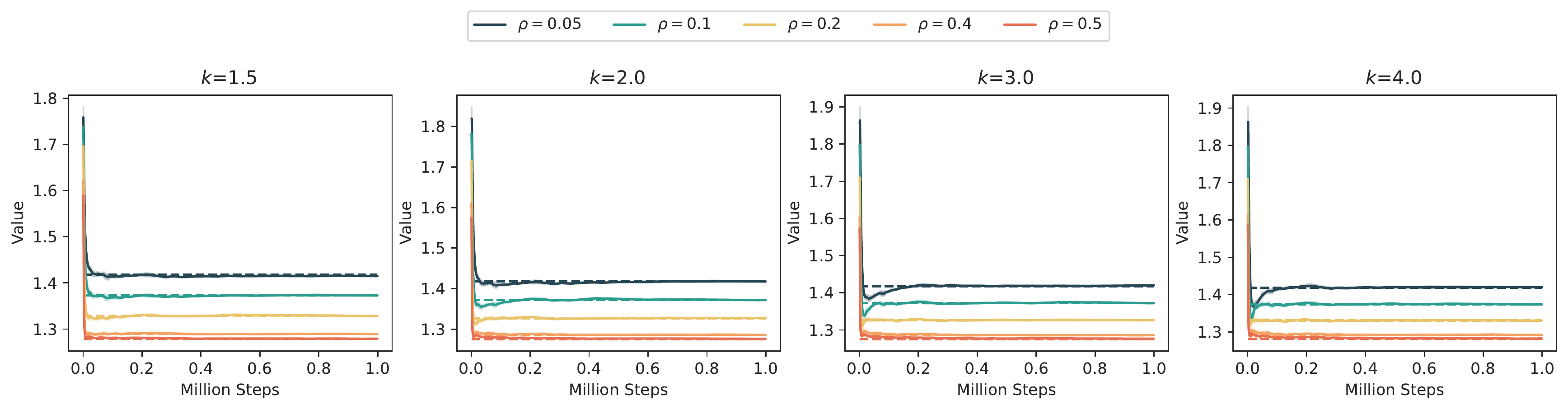}
    \caption{Convergence curve of DR $Q$-learning algorithm to the true DR value under different $\rho$'s and $k$'s. Each curve is averaged over 10 random seeds and shaded by their standard deviation. The dashed line is the optimal robust value with corresponding $k$ and $\rho$.}
    \label{fig:option_curve}
\end{figure}

We present the learning curve of our DR $Q$-learning algorithm with different $\rho$ in Figure~\ref{fig:option_curve}.
Our algorithm can accurately learn the DR value under different $\rho$'s and $k$'s within $0.1$ million steps. 
We compare the sample efficiency of our algorithm with the DR $Q$-learning algorithm in \citet{liu2022distributionally} (referred to as \emph{Liu's}) and the model-based algorithm in \citet{panaganti2022sample} (referred to as \emph{Model}). 
We set a smaller learning rate for Liu's as $\zeta(t) = 1/(1+(1-\gamma) t)$.
The reason is setting the same learning rate $\zeta_3(t)$ for their algorithm would render a much slower convergence performance, which is not fair for comparisons.
We use the recommended choice $\varepsilon = 0.5$ for the sampling procedure in Liu algorithm.
Both DR $Q$-learning and Liu are trained for $5*10^7$ steps per run, while the model-based algorithm is trained for $10^6$ steps per run to ensure sufficient samples for convergence. 
As shown in Figure~\ref{fig:option_comparisons}, 
the model-based approach is the most sample-efficient, converging accurately to the optimal robust value with less than $10^4$ samples. 
Our DR $Q$-learning algorithm is slightly less efficient, using $10^5$ samples to converge.
Liu algorithm is significantly less efficient, using $10^7$ samples to converge.
Note that the model-based approach we compared here is to first obtain samples for each state-action pairs, and then conduct the learning procedure to learn the optimal robust value.
In particular, we need to specify the number of samples for each state-action pair $n$. 
Then the total number of samples used is the sum of all these number, i.e., $S\times A\times n$, whose computation manner is different from that in the model-free algorithms we used where each update requires one or a batch of new samples.

To ensure self-containment, we provide the pseudocode for our implemented Liu algorithm (Algorithm~\ref{alg:discount_dro_Q_cressie_read_Liu}) and the model-based algorithm (Algorithm~\ref{alg:discount_dro_Q_cressie_read_model}) below. These algorithms were not originally designed to solve the ambiguity set constructed by the Cressie-Read family of $f$-divergences.

\subsection{Liu's Algorithm Descriptions}
In this subsection, we provide the pseudo-code for the Liu algorithm, represented in Algorithm~\ref{alg:Liu}. Our intention is to emphasize the differences in algorithmic design between their approach and ours.

Their algorithm, in particular, relies extensively on multi-level Monte Carlo, requiring the sampling of a batch of samples for each state-action pair. Once they estimate the Doubly Robust (DR) value for a specific state-action pair, the samples are promptly discarded and subsequently resampled from a simulator. To summarize, their algorithm exhibits significant distinctions from ours in terms of algorithmic design.

\begin{algorithm}[ht]
    \caption{Distributionally Robust Deep $Q$-learning with Cressie-Read family of $f$-divergences}
    \label{alg:discount_dro_Q_cressie_read_model}
    \begin{algorithmic}[1]
    \STATE {\bfseries Input:} Discount Factor $\gamma$, Radius of robustness $\rho$, Cressie-Read family parameter $k$, $Q$-network target update rate $\tau_{Q}$ and $\eta$-network target update rate $\tau_{\eta}$, mini-batch size $N$, maximum number of iterations $T$, start training timestep $T_{tr}$, training network update frequency $F_{tr}$ and target network update frequency $F_{up}$. 
    \STATE {\bfseries Init:} Two state-action neural networks $Q_{\theta_1}$ and $Q_{\theta_2}$, two dual neural network $\eta_{\theta_1}$ and $\eta_{\theta_2}$,
    $C = (1+k*(k-1) * \rho)^{1/k}$.
    \FOR {for $t=1,\cdots,T$}
        \STATE Observe a state $s_t$ and execute an action $a_t$ using $\epsilon$-greedy policy.
        \IF{$t\ge  T_{tr}$ and $t\% F_{tr}$}
        \STATE Sample a minibatch B with N samples from the replay buffer.
        \STATE Compute next-state target value for $Q$ network
        \begin{align*}
            Q_i = r_t - \gamma C * (\eta_{\theta_1}(s_i, a_i) - \max_{a\in \gA} Q_{\theta_{1}}(s_i, a_i))_+^{k_*},\quad \forall i \in B
        \end{align*}
        and for $\eta$ network
        \begin{align*}
            Q_i^{\prime} = r_t - \gamma C * (\eta_{\theta_2}(s_i, a_i) - \max_{a\in \gA} Q_{\theta_{2}}(s_i, a_i))_+^{k_*},\quad \forall i \in B.
        \end{align*}
        \STATE Update $\theta_1 =\arg\min_{\theta} \sum_i (Q_i - Q_{\theta}(s_i,a_i))^2$.
        \STATE Update $\theta_3 =\arg\max_{\theta} \sum_i Q_i^{\prime}(\theta)$.
        \ENDIF
        \IF{$t\ge T_{tr}$ and $t\% F_{up}$}
        \STATE Update target network $\theta_2 = (1-\tau_Q) \theta_2 + \tau_Q \theta_2$, $\theta_4 = (1-\tau_\eta) \theta_4 + \tau_\eta \theta_3$.
        \ENDIF
    \ENDFOR
    \STATE $t= t+1$
    \end{algorithmic}
    \label{alg:Liu}
\end{algorithm}

\begin{algorithm*}[ht]
    \caption{Distributionally Robust $Q$-learning with Cressie-Read family of $f$-divergences \textbf{with Simulator}}
    \label{alg:discount_dro_Q_cressie_read_Liu}
    \begin{algorithmic}[1]
    \STATE {\bfseries Input:} Exploration rate $\epsilon$, Learning rates $\{\zeta_i(n)\}_{i\in[3]}$, Ambiguity set radius $\rho>0$, parameter $\varepsilon\in (0,0.5)$
    \STATE {\bfseries Init:} $\widehat{Q}(s,a) = 0, \forall (s,a)\in \mathcal{S}\times\mathcal{A}$ 
    \WHILE {Not Converge}
    \FOR {every $(s,a)\in \gS\times\gA$}
    \STATE Sample $N\in \sN$ from $P(N=n) = p_n = \varepsilon(1-\varepsilon)^n$.
    \STATE Draw $2^{N+1}$ samples $\{(r_i, s_i')\}_{i\in [2^{N+1}]}$ from the simulator
    \STATE Compute $\Delta_{N, \rho}^{r}$ via
    \begin{align*}
        \Delta_{N, \rho}^{r} = \sup_{\eta\in \sR} \widehat{\sigma}_k^{r}([2^{N+1}], \eta) - \frac{1}{2}\sup_{\eta\in \sR} \widehat{\sigma}_k^{r}([2^{N}], \eta) - \frac{1}{2}\sup_{\eta\in \sR} \widehat{\sigma}_k^{r}([2^{N}:], \eta),
    \end{align*}
    where 
    \begin{align*}
        \sup_{\eta\in \sR} \widehat{\sigma}_k^{r}(I, \eta) = \sup_{\eta\in \sR} \{-c_k(\rho) [\sum_{i\in I}(\eta - r_i)_+^{k_*}/n]^{\frac{1}{k_*}} + \eta\},
    \end{align*}
    and $[2^{N}] = \{1,2,3,\cdots,2^{N}\}$ and $[2^{N}:]=\{2^{N}, 2^{N} + 1,\cdots, 2^{N+1}\}$.
    \STATE Compute $\Delta_{N, \rho}^{q}(\widehat{Q}_t)$ via
    \begin{align*}
        \Delta_{N, \rho}^{q}(\widehat{Q}_t) = \sup_{\eta\in \sR} \widehat{\sigma}_k^{q}(\widehat{Q}_t, [2^{N+1}], \eta) - \frac{1}{2}\sup_{\eta\in \sR} \widehat{\sigma}_k^{q}(\widehat{Q}_t, [2^{N}], \eta) - \frac{1}{2}\sup_{\eta\in \sR} \widehat{\sigma}_k^{q}(\widehat{Q}_t, [2^{N}:], \eta),
    \end{align*}
    where 
    \begin{align*}
        \sup_{\eta\in \sR} \widehat{\sigma}_k^{q}(\widehat{Q}_t, I, \eta) = \sup_{\eta\in \sR} \{-c_k(\rho) [\sum_{i\in I}(\eta - \max_{a'\in \gA}\widehat{Q}_t(s_i', a'))_+^{k_*}/n]^{\frac{1}{k_*}} + \eta\}.
    \end{align*}
    \STATE Set $R_{\rho}(s,a) = r_1 + \frac{\Delta_{N, \rho}^{r}}{p_N}$.
    \STATE Update $Q$ via 
    \begin{align*}
        \widehat{Q}_{t+1}(s,a) = (1-\zeta_t) \widehat{Q}_{t}(s,a) + \zeta_t \widehat{\gT}_{\rho}(\widehat{Q}_t)(s,a),
    \end{align*}
    where 
    \begin{align*}
        \widehat{\gT}_{\rho}(\widehat{Q}_t)(s,a) = r_1 + \Delta_{N, \rho}^{r} +\gamma(\max_{a'\in \gA} \widehat{Q}_t(s_1, a') + \frac{\Delta_{N, \rho}^{q}(\widehat{Q}_t)}{p_N}).
    \end{align*}
    \ENDFOR
    \STATE $t= t+1$
    \ENDWHILE
    \end{algorithmic}
\end{algorithm*}

\begin{figure}[htbp]
    \centering
    \includegraphics[width=\columnwidth]{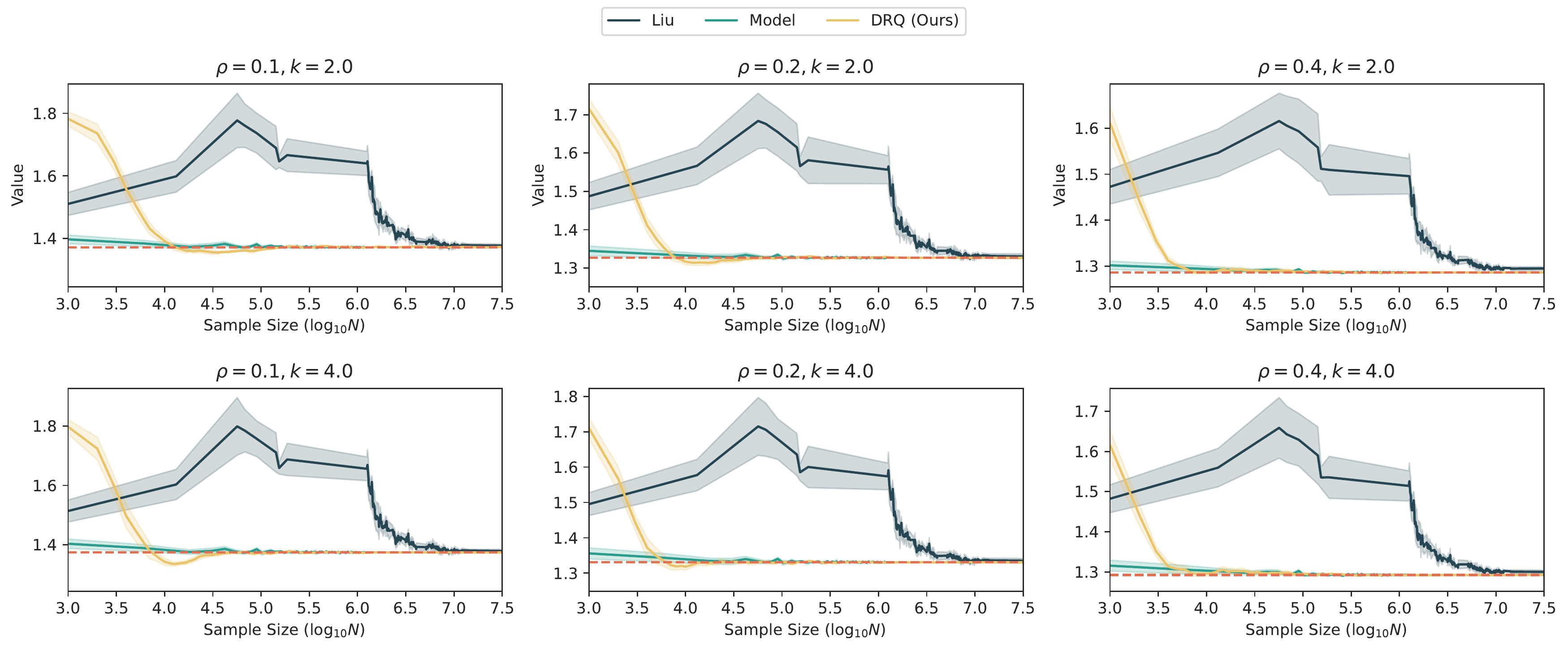}
    \caption{Sample complexity comparisons in  American option environment with other DRRL algorithms. 
    The dashed line is the optimal robust value with corresponding $k$ and $\rho$. The $x$-axis is in log10 scale. Each curve is averaged over 10 random seeds and shaded by their one standard deviation.
    The dashed line is the optimal robust value with corresponding $k$ and $\rho$.}
    \label{fig:option_comparisons}
\end{figure}

\subsection{Practical Experiments}
\label{subsec:practical_exp}
In this section, we provide a comprehensive description of our Deep Distributionally Robust $Q$-learning (DDRQ) algorithm, as illustrated in Algorithm~\ref{alg:discount_dro_Q_cressie_read_model}, along with its experimental setup in the context of CaroPole and LunarLander.

Most of the hyperparameters are set the same for both LunarLander and CartPole. 
We choose Cressie-Read family parameter $k=2$, which is indeed the $\chi^2$ ambiguity set and we set ambiguity set radius as $\rho = 0.3$. 
For RFQI we also use the same $\rho$ for fair comparison. Our replay buffer size is set $1e6$ and the batch size for training is set $4096$. 
Our fast $Q$ and $\eta$ network are update every 10 steps ($F_{tr} = 10$) and the target networks are updated every 500 steps ($F_{up} = 500$). The learning rate for $Q$ network is $2.5 \times 10^{-4}$ and for $\eta$ network is $2.5 \times 10^{-3}$.
The $Q$ network and the $\eta$ network both employ a dual-layer structure, with each layer consisting of 120 dimensions.
For exploration scheme, we choose epsilon-greedy exploration with linearly decay epsilon with ending $\epsilon_{End}$.
The remain parameters tuned for each environments are referred in Table~\ref{tbl:diff_params}.

\begin{table}
    \centering
    \begin{tabular}{c|c|c|c|c}
    \toprule
      Environment & Maximum Training Step $T$ & $\epsilon_{End}$ & $\tau_Q$ & $\tau_\eta$ \\
    \midrule
      CartPole & $1e8$ & $0.05$ & $1$ &  $0.05$ \\
      LunarLander & $3e7$ & $0.2$ & $0.5$ &  $0.1$ \\
    \bottomrule
    \end{tabular}
    \caption{Different Hyperparamers between CartPole and LunarLander}
    \label{tbl:diff_params}    
  \end{table}

\section{Multiple Timescale Convergence}
\label{sec:multi_timescale}
We fix some notations that will be used in the following proof.
For a positive integer $n$, $[n]$ denotes the set $\{1,2,\cdots,n\}$. $\lvert A \rvert $ denotes the cardinality of the set $A$.
We adopt the standard asymptotic notations: for two non-negative sequences ${a_n}$ and ${b_n}$, ${a_n} = O({b_n})$ iff $\lim \sup_{n\rightarrow \infty} a_n/b_n < \infty$. 
$\Delta_{d}$ is the simplex on a $d$ dimensional space, i.e., $\Delta_d = \{x:\sum_{i=1}^d x_i = 1, x_i \ge 0, \forall i\in[d]\}$.
For any vector $x\in \sR^{d}$ and any semi-positive matrix $A\in \sR^{d\times d}$ with $A\succeq 0$, we denote $\lVert x\rVert_{A} \coloneqq \sqrt{x^{\top}Ax}$.
$\lVert \cdot \rVert$ is Euclidean norm.

\subsection{Three Timescales Convergence Analysis}
In this subsection, we outline the roadmap for establishing the a.s. convergence of the Algorithm~\ref{alg:discount_dro_Q_chi}.
For ease of presentation, our analysis is given for the synchronous case, where every entry of the $Q$ function is updated at each timestep. Extension to the asynchronous case, where only one state-action pair entry is updated at each timestep, follows \citet{tsitsiklis1994asynchronous}.
Our approach is to generalize the classic machinery of two-timescale stochastic approximation~\citep{borkar2009stochastic} to a three-timescale framework, and use it to analyze our proposed algorithm.
We rewrite the Algorithm~\ref{alg:discount_dro_Q_chi}  as
\begin{align}
 Z_{n+1} &= Z_{n} + \zeta_1(n)[f(Z_{n}, \eta_n, Q_{n}) + M_{n}^{Z}], \label{eq:ode_Z}\\
 \eta_{n+1} &= \eta_n + \zeta_2(n) [g(Z_{n}, \eta_n, Q_{n}) + \epsilon_n^{\eta}], \label{eq:ode_eta}\\
 Q_{n+1} &= Q_n + \zeta_3(n) [h(Z_{n}, \eta_n, Q_{n}) + \epsilon_n^{Q}]\label{eq:ode_Q}.
\end{align}
Here, we use $Z_n = (Z_{n,1}, Z_{n,2})$ to represent the $Z_{n,1}$ and $Z_{n,2}$ jointly.
To echo with our algorithm, $f=(f_1, f_2)$ and $M_{n}^{Z}=(M_{n,1}^{Z}, M_{n,2}^{Z})$ are defined as,
\begin{align*}
    & f_1(Z_n, \eta_n, Q_n)(s,a) = \sE_{s'}[(\eta_n(s,a)-\max_{a'} Q_n(s',a'))_+^{k_*} - Z_{n,1}(s,a)],\\
    & f_2(Z_n, \eta_n, Q_n)(s,a) = \sE_{s'_n}[(\eta_n(s,a)-\max_{a'} Q_n(s',a'))_+^{k_*-1} - Z_{n,2}(s,a)],\\
    & M_{n,1}^{Z}(s,a) = (\eta_n(s,a)-\max_{a'} Q_n(s',a'))_+^{k_*}- Z_{n,1}(s,a) - f_1(Z_n, \eta_n, Q_n)(s,a),\\
    & M_{n,2}^{Z}(s,a) = (\eta_n(s,a)-\max_{a'} Q_n(s',a'))_+^{k_*-1}- Z_{n,2}(s,a) - f_2(Z_n, \eta_n, Q_n)(s,a).
\end{align*}
In the update of $\eta_{n}$ (Equation~\ref{eq:ode_eta}), $g$ and $\epsilon_n^{\eta}$ are defined as 
\small{
\begin{align*}
    g(Z_n, \eta_n, Q_n)(s,a) &= -c_k(\rho)\mathbb{E}[(\eta_n(s,a)-\max_{a'\in \gA}Q_n(s', a'))_{+}^{k_*}]^{\frac{1}{k_*}-1}\cdot \mathbb{E}[(\eta_n(s,a)-\max_{a'\in \gA}Q_n(s', a'))_{+}^{k_*-1}] +1, \\
    \epsilon_n^{\eta}(s,a) &= -c_k(\rho) Z_{n,1}^{\frac{1}{k_*}-1}(s,a)\cdot Z_{n,2}(s,a) + 1 - g(Z_n, \eta_n, Q_n)(s,a).
\end{align*}
}
Finally in the update of $Q_n$ (Equation~\ref{eq:ode_Q}), $h$ and $\epsilon_n^{Q}$ are defined as
\begin{align*}
    h(Z_n, \eta_n, Q_n)(s,a) &= r(s,a)- \gamma(c_k(\rho)(\mathbb{E}_{P}[(\eta_n(s,a)-\max_{a'\in \gA}Q_n(s', a'))_{+}^{k_*}])^{\frac{1}{k_*}} - \eta_n(s,a)),\\
    \epsilon_n^{Q}(s,a) &= r(s,a)- \gamma(c_k(\rho)Z_{n,1}^{\frac{1}{k_*}}(s,a) - \eta_n(s,a)) - h(Z_n, \eta_n, Q_n)(s,a).
\end{align*}
The algorithm~\ref{alg:discount_dro_Q_chi} approximates the dynamic described by the system of $f$, $g$ and $h$ through samples along a single trajectory, with the resulting approximation error manifesting as martingale noise $M_n^{Z}$ conditioned on some filtration $\gF_n$ and the error terms $\epsilon_n^{\eta}$ and $\epsilon_n^{Q}$. 

To analyze the dynamic of algorithm~\ref{alg:discount_dro_Q_chi}, we first obtain the continuous dynamic of $f,g$, and $h$ using ordinary differential equations (ODEs) analysis. 
The second step is to analyze the stochastic nature of the noise term $M_n^{Z}$ and the error terms $\epsilon_n^{\eta}$ and $\epsilon_n^{Q}$, to ensure that they are negligible compared to the main trend of $f$, $g$, and $h$, which is achieved by the following stepsizes,
\begin{condition}
 \label{as:stepsize}
 The stepsizes $\zeta_i(n), i=1,2,3$ satisfy 
 \begin{align*} &\sum_{n}\zeta_i(n)=\infty,\quad \sum_{n}\zeta_i^2(n) < \infty, \quad \zeta_1(n) = o(\zeta_2(n)), \quad \zeta_2(n) = o(\zeta_3(n)).
 \end{align*}
\end{condition}
\vspace{-0.3cm}
These stepsize schedules satisfy the standard conditions for stochastic approximation algorithms, ensuring that \textbf{(1).} the key quantities in gradient estimator $Z_n$ update on the fastest timescale, \textbf{(2).} the dual variable for the DR problem, $\eta_n$, update on the intermediate timescale; and \textbf{(3).} the $Q$ table updates on the slowest timescale.
Examples of such stepsize are $\zeta_1(n) = \frac{1}{1+n^{0.6}}, \zeta_2(n) = \frac{1}{1 + n^{0.8}}$ and $\zeta_3(n) = \frac{1}{1 + n}$.
Notably, the first two conditions in Condition~\ref{as:stepsize} ensure the martingale noise is negligible.
The different stepsizes for the three loops specificed by the third and fourth conditions ensures that $Z_{n,1}$ and $Z_{n,2}$ are sufficiently estimated with respect to the $\eta_n$ and $Q_n$, and these outer two loops are free from bias or noise in the stochastic approximation sense.

Under Condition~\ref{as:stepsize}, when analyzing the behavior of the $Z_{n}$, the $\eta_{n}$ and the $Q_n$ can be viewed as quasi-static. 
To study the behavior of the fastest loop, we analyze the following ODEs:
\begin{equation}
 \label{eq:ode_1}
 \begin{aligned}
 \dot{Z}(t) = f(Z(t), \beta(t), Q(t)), \quad \dot{\eta}(t) = 0, \quad \dot{Q}(t) = 0,
 \end{aligned}
\end{equation}
and prove that ODEs~(\ref{eq:ode_1}) a.s. converge to $\lambda_1^{''}(\eta, Q)$ for proper $\eta$ and $Q$ and some mapping $\lambda^{''}_1$.
Similarly, $Q_n$ can be viewed as fixed when analyzing the behavior of $\eta_n$, and the corresponding ODEs to understand its behavior are
\begin{equation}
 \label{eq:ode_2}
 \begin{aligned}
 \dot{\eta}(t) = g(\lambda_1^{\prime\prime}(\eta(t), Q(t)), \eta(t), Q(t)), \quad \dot{Q}(t) &= 0.
 \end{aligned}
\end{equation}
By exploiting the dual form of the distributionally robust optimization problem, we can prove these ODEs converge to the set $\{\lambda_1'(Q), \lambda_2'(Q), Q\lvert Q\in V\}$ for some mapping $\lambda_1'$ and $\lambda_2'$ with $V$ is the set containing all the mapping from $\gS$ to $\sR$.
Lastly, we examine the slowest timescale ODE given by
\begin{equation}
 \dot{Q}(t) = h(\lambda_1'(Q(t)), \lambda_2'(Q(t)), Q(t)),
\end{equation}
and employ our analysis to establish the almost sure convergence of Algorithm~\ref{alg:discount_dro_Q_chi} to the globally optimal pair $(Z^{\star}_1, Z^{\star}_2, \eta^{\star}, Q^{\star})$.

\begin{lemma}[Discrete Gronwall inequality]
    \label{lm:discrete_Gronwall}
Let $\left\{x_n, n \geq 0\right\}$ (resp. $\left\{a_n, n \geq\right.$ $0\}$ ) be nonnegative (resp. positive) sequences and $C, L \geq 0$ scalars such that for all $n$,
$$
x_{n+1} \leq C+L\left(\sum_{m=0}^n a_m x_m\right) .
$$
Then for $T_n=\sum_{m=0}^n a_m$,
$$
x_{n+1} \leq C e^{L T_n}.
$$
\end{lemma}

\begin{lemma}[Gronwall inequality]
    \label{lm:Gronwall}
    For continuous $u(\cdot), v(\cdot) \geq 0$ and scalars $C, K, T \geq 0$
$$
u(t) \leq C+K \int_0^t u(s) v(s) d s, \quad \forall t \in[0, T],
$$
implies
$$
u(t) \leq C e^{K \int_0^T v(s) d s}, \quad \forall t \in[0, T] .
$$
\end{lemma}

\subsection{Stability Criterion}
Consider the stochastic approximation scheme $z_n\in \mathbb{R}^N$ given by
$$
z_{n+1}=z_n+a_n\left[g\left(z_n\right)+M_{n+1}\right],
$$
with the following Condition:
\begin{condition}   
    \label{as:stability_1} 
$g: \mathbb{R}^N \rightarrow \mathbb{R}^N$ is Lipschitz.
\end{condition}

\begin{condition}
    \label{as:stability_2} 
The sequence $\left\{a_n\right\} \subset \mathbb{R}$ satisfies $\sum_n a_n=\infty, \sum_n a_n^2<\infty$.
\end{condition}

\begin{condition}    
    \label{as:stability_3} 
$\left\{M_n\right\}$ is a martingale difference sequence with respect to the filtration $\mathcal{F}_n=\sigma\left(z_m, M_m, m \leq n\right)$, there exists $K>0$ such that $E\left[\|M_{n+1}\|^2 \mid \mathcal{F}_n\right] \leq K(1+\|z_n\|^2)$ a.s..
\end{condition}
\begin{condition}    
 The functions $g_d(z)=g(d z) / d, d \geq 1$ satisfy $g_d(z) \rightarrow g_{\infty}(z)$ as $d \rightarrow \infty$ uniformly on compacts for some continuous function $g_{\infty}: \mathbb{R}^N \rightarrow \mathbb{R}^N$. In addition, the ODE
$$
\dot{z}(t)=g_{\infty}(z(t))
$$
has the origin as its globally asymptotically stable equilibrium.
\end{condition}
We then have
\begin{lemma}
    Under Condition~\ref{as:stability_1} to \ref{as:stability_3}, we have 
    $\sup _n\left\|z_n\right\|<\infty$ a.s.
\end{lemma} 

See Section~2.2 and 3.2 in \citet{borkar2009stochastic} for the proof.
As the stability proofs in Section~3.2 of \citet{borkar2009stochastic} are path-wise, we can apply this result to analyze multiple timescales dynamic.

\subsection{Three Timescales Convergence Criterion}
\label{subsec:three_timescales}
Consider the scheme
\begin{align}
&x_{n+1}=x_n+a_n\left[f\left(x_n, y_n, z_n\right)+M_{n+1}^{(1)}\right] \label{eq:first_loop}\\
&y_{n+1}=y_n+b_n\left[g\left(x_n, y_n, z_n\right)+M_{n+1}^{(2)}\right] \label{eq:second_loop}\\
&z_{n+1}=z_n+c_c\left[h\left(x_n, y_n, z_n\right)+M_{n+1}^{(3)}\right] \label{eq:third_loop}
\end{align}
where $f: \mathbb{R}^{d+k+p} \rightarrow \mathbb{R}^d$, $g: \mathbb{R}^{d+k+p} \rightarrow \mathbb{R}^k$, $h:\mathbb{R}^{d+k+p}\rightarrow \mathbb{R}^{p}$, $\{M_n^{(i)}\}, i=1,2,3$ are martingale difference sequences with respect to the $\sigma$-fields $\mathcal{F}_n=\sigma\left(x_m, y_m, M_m^{(1)}, M_m^{(2)}, M_m^{(3)}; m \leq n\right)$, and the $a_n, b_n, c_n$ form decreasing stepsize sequences. 

It is instructive to compare the stochastic update algorithms from Equations~\ref{eq:first_loop} to~\ref{eq:third_loop} with the following o.d.e.,
\begin{align*}
    \dot{x}(t) &= \frac{1}{a} f(x(t), y(t), z(t)),\\
    \dot{y}(t) &= \frac{1}{b} g(x(t), y(t), z(t)),\\
    \dot{z}(t) &= \frac{1}{c} h(x(t), y(t), z(t)),
\end{align*}
in the limit that $a, b, c\rightarrow 0$ and $a = o(b)$, $c = o(b)$.

We impose the following conditions, which are necessary for the a.s. convergence for each timescale itself and  are commonly used in the literature of stochastic approximation algorithms, e.g., \cite{borkar2009stochastic}.
\begin{condition}
    \label{as:Lipschitz}
    $f$ and $g$ is $L$-Lipschitz map for some $0<L<\infty$ and $h$ is bounded.
\end{condition}

\begin{condition}
    \label{as:learning_rate}
    \begin{align*}
        \sum_{n}a_n = \sum_{n}b_n = \sum_{n}c_n = \infty, \sum_{n}(a_n^2+ b_n^2 + c_n^2)<\infty,\text{ and } b_n = o(a_n), c_n = o(b_n).
    \end{align*}
\end{condition}

\begin{condition}
    \label{as:martingale}
    For $i=1,2,3$ and $n\in \sN^+$, $\{M_{n}^{(i)}\}$ is a martingale differeence sequence with respect to the increasing family of $\sigma$-fields $\gF_n$.
    Furthermore, there exists some $K>0$, such that for $i=1,2,3$ and $n\in \sN^+$,
    \begin{align*}
        \sE[\lVert M_{n+1}^{(i)} \rVert^2\lvert \gF_n]\le K(1+\lVert x_n\rVert^2 + \lVert y_n\rVert^2 + \lVert z_n\rVert^2).
    \end{align*}
\end{condition}

\begin{condition}
    \label{as:bounded}
    $\sup_{n}(\lVert x_n\rVert + \lVert y_n\rVert + \lVert z_n\rVert)<\infty$, a.s..
\end{condition}

\begin{condition}
    \label{as:first_ode}
    For each $y\in \sR^{k}$ and $z\in \sR^{p}$, $\dot{x}(t) = f(x(t), y, z)$ has a globally asymptotically stable equilibrium $\lambda_1(y,z)$, where $\lambda_1:\gR^{k+p}\rightarrow \gR^{d}$ is a $L$-Lipschitz map for some $L>0$.
\end{condition}

\begin{condition}
    \label{as:second_ode}
    For each $z\in \sR^{p}$, $\dot{y}(t) = g(\lambda_1(y(t),z), y(t), z)$ has a globally asymptotically stable equilibrium $\lambda_2(z)$, where $\lambda_2:\gR^{p}\rightarrow \gR^{k}$ is a $L$-Lipschitz map for some $L>0$.
\end{condition}

\begin{condition}
    \label{as:third_ode}
    $\dot{z}(t) = h(\lambda_1(z(t)), \lambda_2(z(t)), z(t))$ has a globally asymptotically stable equilibrium $z^{\star}$.
\end{condition}

Conditions~\ref{as:Lipschitz}, \ref{as:learning_rate}, \ref{as:martingale} and \ref{as:bounded} are necessary for the a.s. convergence for each timescale itself.
Moreover, Condition~\ref{as:bounded} itself requires Conditions like ~\ref{as:Lipschitz}, \ref{as:learning_rate}, \ref{as:martingale}, with an extra condition like Condition~\ref{as:stability_3}.
Instead, we need to prove the boundedness for each timescale, thus the three timescales version is as follow
\begin{condition}
\label{eq:three-timescale_stability}
The ODE
\begin{align*}
    \dot{z}(t)&=f_{\infty}(x(t), y, z)\\
    \dot{y}(t)&=g_{\infty}(\lambda_1(y(t), z), y(t), z)\\
    \dot{z}(t)&=h_{\infty}(\lambda_1(z(t)), \lambda_2(z(t)), z(t))
\end{align*}
all have the origin as their globally asymptotically stable equilibrium for each $y\in \gR^{k}$ and $z\in \gR^{p}$, where 
\begin{align*}
    f_{\infty} = \lim_{d\rightarrow \infty} \frac{f(d x)}{d}, \quad g_{\infty} = \lim_{d\rightarrow \infty} \frac{g(d x)}{d}, \text{ and }
    h_{\infty} = \lim_{d\rightarrow \infty} \frac{h(d x)}{d}.
\end{align*}
\end{condition}

We have the following results, which appears as a three timescales extension of Lemma 6.1 in \citet{borkar2009stochastic} and serves as a auxiliary lemma for the our a.s. convergence.
\begin{lemma}
    \label{lm:converge_condition}
    Under the conditions~\ref{as:Lipschitz}, \ref{as:learning_rate}, \ref{as:martingale} and \ref{as:bounded}.
    $(x_n, y_n, z_n)\rightarrow \{\lambda_1^{\prime}(z),\lambda_2^{\prime}(z), z: z\in \gR^{p}\}$ a.s..
\end{lemma}

\begin{proof}
    Rewrite Equations~\ref{eq:second_loop} and \ref{eq:third_loop} as 
    \begin{align*}
        y_{n+1}&=y_n+a_n\left[\epsilon_{1,n}+M_{n+1}^{(2)'}\right]\\
        z_{n+1}&=z_n+a_n\left[\epsilon_{2,n}+M_{n+1}^{(3)'}\right],
    \end{align*}
    where $\epsilon_{1,n} = \frac{b_n}{a_n}g(x_n,y_n,z_n)$, $\epsilon_{2,n} = \frac{c_n}{a_n}h(x_n,y_n,z_n)$, $M_{n+1}^{(2)'} = \frac{b_n}{a_n}M_{n+1}^{(2)}$, $M_{n+1}^{(3)'} = \frac{c_n}{a_n}M_{n+1}^{(3)}$.
    Note that $\epsilon_{1,n}, \epsilon_{2,n} \rightarrow 0$ as $n\rightarrow \infty$.
    Consider them as the special case in the third extension in Section 2.2 in \citet{borkar2009stochastic} and then we can conclude that $(x_n,y_n,z_n)$ converges to the internally chain transitive invariant sets of the o.d.e.,
    \begin{align*}
        \dot{x}(t) &= h(x(t), y(t), z(t))\\
        \dot{y}(t) &= 0\\
        \dot{z}(t) &= 0,
    \end{align*}
    which implies that $(x_n, y_n, z_n)\rightarrow \{\lambda_1^{\prime}(y, z), y, z: y\in \gR^{k}, z\in \gR^{p}\}$.

    Rewrite Equation~\ref{eq:third_loop} again as 
    \begin{align*}
        z_{n+1}&=z_n+b_n\left[\epsilon_{2,n}^{\prime}+M_{n+1}^{(3)''}\right],
    \end{align*}
    where $\epsilon_{2,n}^{\prime} = \frac{c_n}{b_n}h(x_n, y_n, z_n)$ and $M_{n+1}^{(3)''} = \frac{c_n}{b_n}M_{n+1}^{(3)}$.
    We use the same extension again and can conclude that $(x_n, y_n, z_n)$ converges to the internally chain transitive invariant sets of the o.d.e., 
    \begin{align*}
        \dot{y}(t) &= g(\lambda_1^{\prime}(y(t)), y(t), z(t))\\
        \dot{z}(t) &= 0.
    \end{align*}
    Thus $(x_n, y_n, z_n)\rightarrow \{\lambda_1(y), \lambda_2(z), z: z\in \gR^{p}\}$.
\end{proof}

\begin{theorem}
    Under the Condition~\ref{as:Lipschitz} to \ref{eq:three-timescale_stability}, $(x_n, y_n, z_n)\rightarrow (\lambda_1(z^{*}), \lambda_2(z^{*}), z^{*})$.
\end{theorem}
\begin{proof}
    Let $t(0)=0$ and $t(n) = \sum_{i=0}^{n-1}c_i$ for $n\ge 1$.
    Define the piecewise linear continuous function $\tilde{z}(t), t\ge 0$ where $\tilde{z}(t(n)) = z_n$ and $\tilde{z}(t) = \frac{t(n+1)-t}{t(n+1) - t(n)} z_{n+1} + \frac{t-t(n)}{t(n+1) - t(n)} z_{n}$ for $t\in [t(n), t(n+1)]$ with any $n\in N$.
    Let $\psi_n = \sum_{i=0}^{n-1}c_i M_{i+1}^{(3)}, n\in \sN^{+}$.
    For any $ t\ge 0$, denote $[t] = \max\{s(n): s(n)\le t\}$. Then for $n,m\ge 0$, we have 
    \begin{align}
        \tilde{z}(t(n+m)) 
        &= \tilde{z}(t(n)) + \sum_{k=1}^{m-1} c_{n+k} h(x_{n+k}, y_{n+k}, z_{n+k}) + (\psi_{m+n+1} - \psi_{n})\notag \\
        &= \tilde{z}(t(n)) + \int_{t(n)}^{t(n+m)} h(\lambda_1(z(s)), \lambda_2(z(s)), z(s)) ds \notag \\
        &+ \int_{t(n)}^{t(n+m)} (h(\lambda_1(z([s])), \lambda_2(z([s])), z([s])) - h(\lambda_1(z(s)), \lambda_2(z(s)), z(s)))ds \notag \\
        &+ \sum_{k=0}^{m-1} c_{n+k}(h(x_{n+k}, y_{n+k}, z_{n+k}) - h(\lambda_1(z_{n+k}), \lambda_2(z_{n+k}), z_{n+k})) \notag\\
        &+ (\psi_{n+m+1} - \psi_{n}). \label{eq:z_tilde}  
    \end{align}

    We further define $z^{t(n)}(t)$ as the trajectory of $\dot{z}(t) = g(\lambda_1(z(t)), \lambda_2(z(t)), z(t))$ with $z^{t(n)}(t(n)) = \tilde{z}(t(n))$.
    \begin{align}
        z^{t(n)}(t(n+m)) &= \tilde{z}(t(n)) + \int_{t(n)}^{t(n+m)}h(\lambda_1(z^{t(n)}(s)), \lambda_2(z^{t(n)}(s)), z^{t(n)}(s))ds
        \label{eq:ode}.
    \end{align}
    Taking the difference between Equation~\ref{eq:z_tilde} and the Equation~\ref{eq:ode} we have 
    \begin{align*}
        &\lvert \tilde{z}(t(n+m)) - z^{t(n)}(t(n+m))\rvert\\
        &=  \underbrace{\sum_{k=0}^{m-1}c_{n+k}(h(\lambda_1(\tilde{z}(t+k)), \lambda_2(\tilde{z}(t+k)), \tilde{z}(t+k))-h(\lambda_1(z(t(n+k))), \lambda_2(z(t(n+k))), z(t(n+k))))} \\
        &\quad +\underbrace{\lvert  \int_{t(n)}^{t(n+m)} (h(\lambda_1(z([t])), \lambda_2(z([t])), z([t])) - h(\lambda_1(z(s)), \lambda_2(z(s)), z(s)))ds \rvert}_{\operatorname{I}} \\
        &\quad + \underbrace{\lvert  \sum_{k=1}^{m-1} c_{n+k}(h(x_{n+k}, y_{n+k}, z_{n+k}) - h(\lambda_1(z_{n+k}), \lambda_2(z_{n+k}), z_{n+k})) \rvert}_{\operatorname{II}} \\
        &\quad+ \underbrace{\lvert \psi_{n+m+1} - \psi_{n} \rvert}_{\operatorname{III}}.
    \end{align*}
We analyze the I term. For notation simplicity we ignore the supsript $t(n)$.
\begin{align}
    & \lvert h(\lambda_1(z([t])), \lambda_2(z([t])), z([t])) - h(\lambda_1(z(t)), \lambda_2(z(t)), z(t))\rvert\notag\\
    &= \lvert (h(\lambda_1(z([t])), \lambda_2(z([t])), z([t])) - h(\lambda_1(z([t])), \lambda_2(z([t])), z(t)))\rvert\notag\\
    &\quad + \lvert (h(\lambda_1(z([t])), \lambda_2(z([t])), z(t)) - h(\lambda_1(z([t])), \lambda_2(z([t])), z([t])))\rvert\notag\\
    &= \lvert (h(\lambda_1(z([t])), \lambda_2(z([t])), z([t])) - h(\lambda_1(z([t])), \lambda_2(z(t)), z(t)))\rvert\notag\\
    &\quad + \lvert h(\lambda_1(z([t])), \lambda_2(z([t])), z(t)) - h(\lambda_1(z([t])), \lambda_2(z([t])), z(t)) \rvert\notag\\
    &\quad + \lvert(h(\lambda_1(z([t])), \lambda_2(z([t])), z(t)) - h(\lambda_1(z([t])), \lambda_2(z([t])), z([t])))\rvert.
\end{align}
By the Lipschitzness of the $h$ we have 
\begin{align*}
    \lVert h(x) - h(0)\rVert \le L\lVert x\rVert,
\end{align*}
which implies 
\begin{align*}
    \lVert h(x)\rVert \le \lVert h(0) \rVert + L\lVert x\rVert.
\end{align*}
\begin{align*}
    \lVert z^{t(n)}(t)\rVert&\le \lVert \tilde{z}(s) \rVert + \int_{s}^{t}\lVert h(z^{t(n)}(s))\rVert ds\\
    &\le \lVert \tilde{z}(s) \rVert + \int_{s}^{t} (\lVert h(0)\rVert +  L\lVert z^{t(n)}(s) \rVert ) ds\\
    &\le ( \lVert \tilde{z}(s) \rVert + \lVert h(0)\rVert T) + L \int_{s}^t \lVert z^{t(n)}(s) \rVert ds.
\end{align*}
By Gronwall's inequality (Lemma~\ref{lm:Gronwall}), we have 
\begin{align*}
    \lVert z^{t(n)}(t)\rVert \le (C + \lVert h(0)\rVert T)e^{LT},\quad \forall t\in [t(n), t(n+m)].
\end{align*}
Thus for all $t\in [t(n), t(n+m)]$, we have 
\begin{align*}
    \lVert h(\lambda_1(z^{t(n)}(t)), \lambda_2(z^{t(n)}(t)), z^{t(n)}(t)) \rVert \le C_T \coloneqq  \lVert h(0)\rVert + L (C + \lVert h(0)\rVert T)e^{LT}<\infty, a.s..
\end{align*}
For any $k\in[m-1]$ and $t\in [t(n+k), t(n+k+1)]$, 
\begin{align*}
    \lVert z^{t(n)}(t) - z^{t(n)}(t(n+k))\rVert &\le \lVert \int_{t(n+k)}^{t} h(\lambda_1(z^{t(n)}(s)), \lambda_2(z^{t(n)}(s)), z^{t(n)}(s))ds\rVert\\
    &\le C_T(t-t(n+k))\\
    &\le C_T a(n+k),
\end{align*}
where the last inequality is from the construction of $\{t(n):n\in \sN^{+}\}$.   
Finally we can conclude 
\begin{align*}
    &\lVert \int_{t(n)}^{t(n+m)}(h(\lambda_1(z([s])), \lambda_2(z([s])), z(s)) - h(\lambda_1(z([s])), \lambda_2(z([s])), z([s]))) ds \rVert \\
    &\le \int_{t(n)}^{t(n+m)} L\lVert z(s) - z([s])\rVert ds\\
    &= L \sum_{k=0}^{m-1}\int_{t(n+k)}^{t(n+k-1)}\lVert z(s) - z(t(n+k)) \rVert ds\\
    &\le C_T L \sum_{k=0}^{m-1} c_{n+k}^2\\
    &\le C_T L \sum_{k=0}^{\infty} c_{n+k}^2\rightarrow 0, a.s..
\end{align*}

For the III term, it converges to zero from the martingale convergence property.

Subtracting equation~\ref{eq:z_tilde} from~\ref{eq:ode} and take norms, we have 
\begin{align*}
    &\lVert \tilde{z}(t(n+m)) - z^{t(n)}(t(n+m))\rVert\\
    &\le L \sum_{i=0}^{m-1} c_{n+i} \lVert \tilde{z}(t(n+i)) - z^{t(n)}(t(n+i))\rVert\\
    & \quad + C_T L \sum_{k\ge 0}c_{n+k}^{2} + \sup_{k\ge 0} \lVert \delta_{n, n+k} \rVert, a.s..
\end{align*}
Define $K_{T,n} = C_T L \sum_{k\ge 0} c_{n+k}^2 + \sup_{k\ge 0}\lVert \delta_{n, n+k}\rVert$.
Note that $K_{T,n} \rightarrow 0$ a.s. $n\rightarrow \infty$.
Let $u_i = \lVert \tilde{x}(t(n+i)) - x^{t(n)}(t(n+i))\rVert$.
Thus, above inequality becomes
\begin{align*}
    u_m \le K_{T,n} + L \sum_{i=0}^{m-1} c_{n+i} u_i.
\end{align*}

Thus the above inequality becomes 
\begin{align*}
    z(t(n+m))\le K_{T,n} + L\sum_{k=0}^{m-1}c_k z(t(n+k)).
\end{align*}
Note that $u_0=0$ and $\sum_{i=0}^{m-1}b_i\le T$, then using the discrete Gronwall lemma (Lemma~\ref{lm:discrete_Gronwall}) we have 
\begin{align*}
    \sup_{0\le i \le m}u_i\le K_{T,n}e^{LT}.
\end{align*}

Following the similar logic as in Lemma 1 in \citet{borkar2009stochastic}, we can extend the above result to the case $\lVert \tilde{z}(t) - z^{t(n)}(t)\rVert\rightarrow 0$ where $t\in [0, T]$.

Then using the proof of Theorem 2 of Chapter 2 in \citet{borkar2009stochastic}, we get $z_n \rightarrow z^{*}$ a.s. and thus by Lemma~\ref{lm:converge_condition} the proof can be concluded.
\end{proof}

\section{Convergence of the DR $Q$-learning Algorithm} 
\label{subsec:proof_discount_chi}
Before we start the proof of the DR $Q$-learning algorithm, we first introduce the following lemma.
\begin{lemma}
Denote 
$\eta^* = \arg\max_{\eta} \sigma_k(X, \eta) = -c_k(\rho) \sE_P[(\eta - X)_+^{k_*}]^{\frac{1}{k_*}} + \eta$. Given that $X(\omega)\in [0, M]$, then we have $\eta*\in [0, \frac{c_k(\rho)}{c_k(\rho) - 1}M]$.
\end{lemma}
\begin{proof}
Note that for $\eta = \min_{\omega} X(\omega)$, $
-c_k(\rho) \sE_P[(\eta - X)_+^{k_*}]^{\frac{1}{k_*}} + \eta = \min_{\omega} X(\omega)\ge 0$.
Also we know that when $\eta\ge \frac{c_k(\rho)}{c_k(\rho) - 1}M$,
\begin{align*}
    &-c_k(\rho) \sE_P[(\eta - X)_+^{k_*}]^{\frac{1}{k_*}} + \eta\\
    \le &-c_k(\rho) \sE_P[(\eta - M)_+^{k_*}]^{\frac{1}{k_*}} + \eta\\
    = &-c_k(\rho) (\eta - M) + \eta\\
    \le &0.
\end{align*}
Then we can conclude that $\eta^*\le \frac{c_k(\rho)}{c_k(\rho) - 1}M$. Moreover, as $X(\omega)\ge 0$, we know $\sigma_k(X, 0) = 0$, which concludes that $\eta*\in [0, \frac{c_k(\rho)}{c_k(\rho) - 1}M]$.
\end{proof}

Note that $Q_n\in [0, \frac{1}{1-\gamma}]$ when reward is bounded by $[0, 1]$. Thus $M = \frac{1}{1-\gamma}$ in our case and then we denote $ \overline{\eta} = \frac{c_k(\rho)}{c_k(\rho) - 1}M$.
Now we are ready to prove the convergence of the DR $Q$-learning algorithm.
For theoretical analysis, we consider the clipping version of our DR $Q$-learning algorithm.   
\begin{proof}[Proof of Theorem~\ref{thm:convergence_discount_dro_chi}]
We define the filtration generated by the historical trajectory,
\begin{align*}
    \gF_n = \sigma(\{(s_t, a_t, s_t^{\prime}, r_t)\}_{t\in[n-1]}, s_n, a_n).
\end{align*}
In the following analysis, we fix for a $(s,a)\in \gS\times \gA$ but ignore the $(s,a)$ dependence for notation simplicity. 
Following the roadmap in Section 3.4, we rewrite the algorithm as
\begin{align}
    Z_{n+1,1} &= Z_{n,1} + \zeta_1(n)[f_1(Z_{n,1}, Z_{n,2}, \eta_n, Q_{n}) + M_{n+1}^{(1)}], \label{eq:ode_Z1_chi}\\
    Z_{n+1,2} &= Z_{n,2} + \zeta_1(n)[f_2(Z_{n,1}, Z_{n,2}, \eta_n, Q_{n}) + M_{n+1}^{(2)}], \label{eq:ode_Z2_chi}\\
    \eta_{n+1} &= \Gamma_\eta \left[\eta_n + \zeta_2(n) f_3(Z_{n,1}, Z_{n,2}, \eta_n, Q_{n})\right], \label{eq:ode_eta_chi}\\
    Q_{n+1} &= \Gamma_Q [Q_n + \zeta_3(n) [f_4(Z_{n,1}, Z_{n,2}, \eta_n, Q_{n})]].
\end{align}

Here for theoretical analysis, we add a clipping operator $\Gamma_\eta(x) = \min(\max(x, 0), \overline{\eta})$ and $\Gamma_Q(x) = \min(\max(x, 0), M)$ compared with the algorithm presented in the main text.

We first proceed by first identifying the terms in Equation~\ref{eq:ode_Z1_chi} and \ref{eq:ode_Z2_chi} and studying the corresponding ODEs
\begin{align*}
    \dot{Q}(t) &= 0,\\
    \dot{\eta}(t) &= 0,\\
    \dot{Z}_1(t) &= f_1(Z_1(t), Z_2(t), \eta(t), Q(t)).\\
    \dot{Z}_2(t) &= f_2(Z_1(t), Z_2(t), \eta(t), Q(t)).\\
\end{align*}
As $f_1$ and $f_2$ is in fact irrelavant to the $Z_2$ and $Z_1$, we analyze their equilibria seperately. For notation convenience, we denote $y_n(s) = \max_{a'\in \gA}Q_n(s,a')$.

For ODE~\ref{eq:ode_Z1_chi} and each $\eta_n\in \sR, Q_n \in \gS\times \gA \rightarrow \sR$, it is easy to know there exists a unique global asymtotically stable equilibrium $Z_{n,1}^{\star} = \lambda_1(\eta_n, y_n) = \sE[(\eta_n - y_n)_{+}^{k_*}]$.
Similarly, For ODE~\ref{eq:ode_Z2_chi} and each $\eta_n\in \sR, Q_n \in \gS\times \gA \rightarrow \sR$, there exists a unique global asymtotically stable equilibrium $Z_{n,2}^{\star} = \lambda_2(\eta, y) = \sE[(\eta_n - y_n)_{+}^{k_* - 1}]$.

Second, $M_{n+1}^{(1)} = (\eta_n - y_n)_{+}^{y_*} - \sE[(\eta_n - y_n)_{+}^{y_*}]$ and $M_{n+1}^{(2)} = (\eta_n - y_n)_{+}^{y_*-1} - \sE[(\eta_n - y_n)_{+}^{y_*-1}]$.
Note that for any $(s,a)\in \gS\times \gA$, $\eta_n(s,a)\le \overline{\eta}$, $y_n(s')\le M$ and $M\le \overline{\eta}$.
Thus  $\lvert (\eta_n(s,a) - y_n(s'))_{+}^{y_*}\rvert \le \overline{\eta}^{y_*}$, which leads to $\lvert M_{n+1}^{(1)}(s,a)\rvert = \lvert (\eta_n(s,a) - y_n(s'))_{+}^{y_*} - \sE[(\eta_n(s,a) - y_n(s'))_{+}^{y_*}]\rvert\le \overline{\eta}^{k_*}$.

Since $\lVert y_n \rVert_{\infty}\le \lVert Q_n \rVert_{\infty}$ and $(x-y)_+^2\le x^2 + y^2$ for any $x,y$, we have, 
\begin{align*}
    &\sE[\lVert M_{n+1}^{(1)}\rVert^2 \lvert \gF_n] \\
    &=\sE[\lVert (\eta_n - y_n)_{+}^{y_*} - \sE[(\eta_n - y_n)_{+}^{y_*}]\rVert^2\lvert \gF_n]\\
    &\le K_1(1+\lVert Z_{n,1} \rVert^2+\lVert Z_{n,2} \rVert^2+\lVert Q_n\rVert^2 + \lVert \eta_n\rVert^2),
\end{align*}
where $K_1 = S \overline{\eta}^{2k_*}$. Similarly, we can conclude that $\sE[\lVert M_{n+1}^{(2)}\rVert^2 \lvert \gF_n]\le K_2(1+\lVert Z_{n,1} \rVert^2+\lVert Z_{n,2} \rVert^2+\lVert Q_n\rVert^2 + \lVert \eta_n\rVert^2)$ for some $K_2  = S \overline{\eta}^{2(k_*-1)}$.

Next we analyze the second loop.
\begin{align*}
    \dot{Q}(t) &= 0,\\
    \dot{\eta}(t) &= \Gamma_\eta[f_3(\lambda_1(\eta(t), Q(t)), \lambda_2(\eta(t), Q(t)), \eta(t), Q(t))],
\end{align*}
where 
\begin{align*}
    f_3(\lambda_1(\eta, Q), \lambda_2(\eta, Q), \eta, Q) = -c_k(\rho) \lambda_1(\eta, Q)^{\frac{1}{k_*}-1} \lambda_2(\eta, Q) + 1.
\end{align*}

The global convergence point is $\eta^*(t) = \arg \max_{\eta\in [0, \overline{\eta}]}\{\sigma_{k}(Q, \eta)\} = \arg \max_{\eta\in \sR}\{\sigma_{k}(Q, \eta)\}$.

Finally we arrive to the outer loop, i.e., 
\begin{align*}
    \dot{Q}(t) = \Gamma_Q [f_4(\lambda_1(Q(t)), \lambda_2(Q(t)), \lambda_3(Q(t)), Q(t))].
\end{align*}
By using the dual form of Cressie-Read Divergence (Lemma~\ref{lemma:cressie_dual}), we know that this is equivilant to 
\begin{align*}
    \dot{Q}(t) = r + \gamma \inf_{P\in \gP}\sE_{P}[\max_{a'}Q(s',a')] - Q(t),
\end{align*}
for ambiguity set using Cressie-Read of $f$ divergence.

Denote $H(t) = r + \gamma \inf_{P\in \gP}\sE_{P}[\max_{a'}Q(s',a')]$ and thus 
we can rewrite the above ODE as 
\begin{align*}
    \dot{Q}(t) = H(t) - Q(t).
\end{align*}

 Following , we consider its infity version, i.e., $H^{\infty}(t) = \lim_{c\rightarrow \infty}H(ct)/c$.
\begin{align*}
    \dot{Q}(t) = \gamma \inf_{P\in \gP}\sE_{P}[\max_{a'}Q(s',a')] - Q(t).
\end{align*}
This is a contraction by Theorem 3.2 in \citet{iyengar2005robust}.
By the proof in Section 3.2 in \citet{borkar2000ode}, we know the contraction can lead to the global unique equilibrium point in the ode.
Thus we finish verifying all the conditions in Section~\ref{subsec:three_timescales}, which can lead to the desired result.
\end{proof}

\end{document}